\documentclass[runningheads]{llncs}
\usepackage{graphicx}
\usepackage{subfigure}
\usepackage{amsmath,amssymb} 
\usepackage{color}
\usepackage[noend]{algpseudocode}
\usepackage{algorithmicx,algorithm}

\newcommand{\repeatthanks}{\textsuperscript{\thefootnote}}

\begin{document}
\pagestyle{headings}
\mainmatter
\def\ECCVSubNumber{5299} 

\title{Globally-Optimal Event Camera\\Motion Estimation}


\titlerunning{Globally-Optimal Event Camera Motion Estimation}

\author{Xin Peng\inst{1,2}\thanks{indicates equal contribution.} \and
Yifu Wang\inst{3}\repeatthanks \and
Ling Gao\inst{1}\repeatthanks \and Laurent Kneip\inst{1,4}}

\authorrunning{X. Peng, Y. Wang, L. Gao, L. Kneip.}

\institute{Mobile Perception Lab, SIST, ShanghaiTech University \and
Shanghai Institute of Microsystems and Information Technology,\\ Chinese Academy of Sciences \and Australian National University \and Shanghai Engineering Research Center of Intelligent Vision and Imaging \\
}


\maketitle

\begin{abstract}
Event cameras are bio-inspired sensors that perform well in HDR conditions and have high temporal resolution. However, different from traditional frame-based cameras, event cameras measure asynchronous pixel-level brightness changes and return them in a highly discretised format, hence new algorithms are needed. The present paper looks at fronto-parallel motion estimation of an event camera. The flow of the events is modeled by a general homographic warping in a space-time volume, and the objective is formulated as a maximisation of contrast within the image of unwarped events. However, in stark contrast to prior art, we derive a globally optimal solution to this generally non-convex problem, and thus remove the dependency on a good initial guess. Our algorithm relies on branch-and-bound optimisation for which we derive novel, recursive upper and lower bounds for six different contrast estimation functions. The practical validity of our approach is supported by a highly successful application to AGV motion estimation with a downward facing event camera, a challenging scenario in which the sensor experiences fronto-parallel motion in front of noisy, fast moving textures.


\end{abstract}


\section{Introduction}

Camera motion estimation is an important technology with many applications in automation, smart transportation, and assistive technologies. However, despite the fact that a certain level of maturity has already been reached, we keep facing challenges in scenarios with high dynamics, low texture distinctiveness, or challenging illumination conditions \cite{fuentes2015visual,cadena2016past}. Event cameras---also called dynamic vision sensors---present an interesting alternative in this regard, as they pair HDR with high temporal resolution. The potential advantages and challenges behind event-based vision are well explained by the original work of Brandli~et~al.~\cite{brandli2014240} as well as the recent survey by Gallego~et~al.~\cite{gallego2019event}.

Our work considers fronto-parallel motion estimation of an event camera. The flow of the events is hereby modelled by a general homographic warping in a space-time volume, and motion may be estimated by maximisation of contrast in the image of unwarped events~\cite{gallego2018unifying}. Various reward functions that maximise contrast have been presented and analysed in the recent works of Gallego~et~al.~\cite{gallego2019focus} and Stoffregen~and~Kleeman~\cite{stoffregen2019event1}, and successfully used for solving a variety of problems with event cameras such as optical flow~\cite{zhu2017event,gallego2018unifying,stoffregen2018simultaneous,DBLP:journals/corr/abs-1809-08625,zhu2019unsupervised,zhu2018ev}, segmentation~\cite{stoffregen2018simultaneous,stoffregen2019event,mitrokhinevent}, 3D reconstruction~\cite{rebecq2018emvs,zhu2018realtime,zhu2019unsupervised,DBLP:journals/corr/abs-1809-08625}, and motion estimation~\cite{gallego2017accurate,gallego2018unifying}. Our work focuses on the latter problem of camera motion estimation. However---different from many of the aforementioned works---we propose the first globally optimal solution to the underlying contrast maximisation problem, an important point given its generally non-convex nature.

Our detailed contributions are as follows:
\begin{itemize}
\item We solve the global maximisation of contrast functions via Branch and Bound.
\item We derive bounds for six different contrast estimation functions. The bounds are furthermore calculated recursively, which enables efficient processing.
\item We successfully apply this strategy to Autonomous Ground Vehicle (AGV) planar motion estimation with a downward facing event camera (cf. Figure \ref{fig:robot and frame}), a problem that is complicated by motion blur, challenging illumination conditions, and indistinctive, noisy textures. We prove that using an event camera can solve these challenges, hence outperforming alternatives given by regular cameras.
\end{itemize}

\begin{figure}[t]
\centering
\subfigure[AGV]
{
\includegraphics[width=0.205\textwidth]{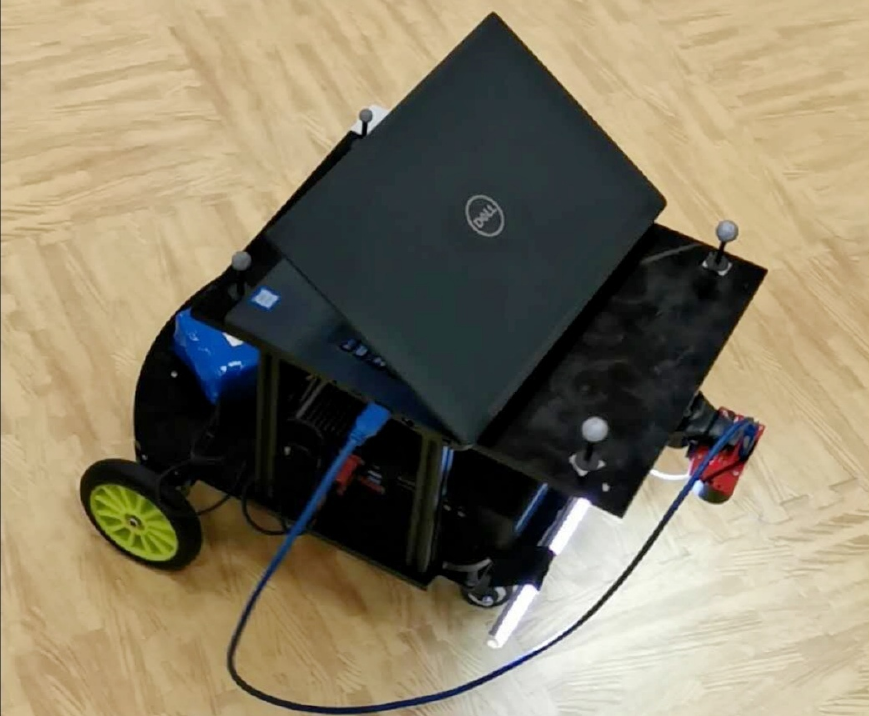}
}
\subfigure[wood grain foam]{
	\includegraphics[width=0.23\textwidth , height=0.17\textwidth]{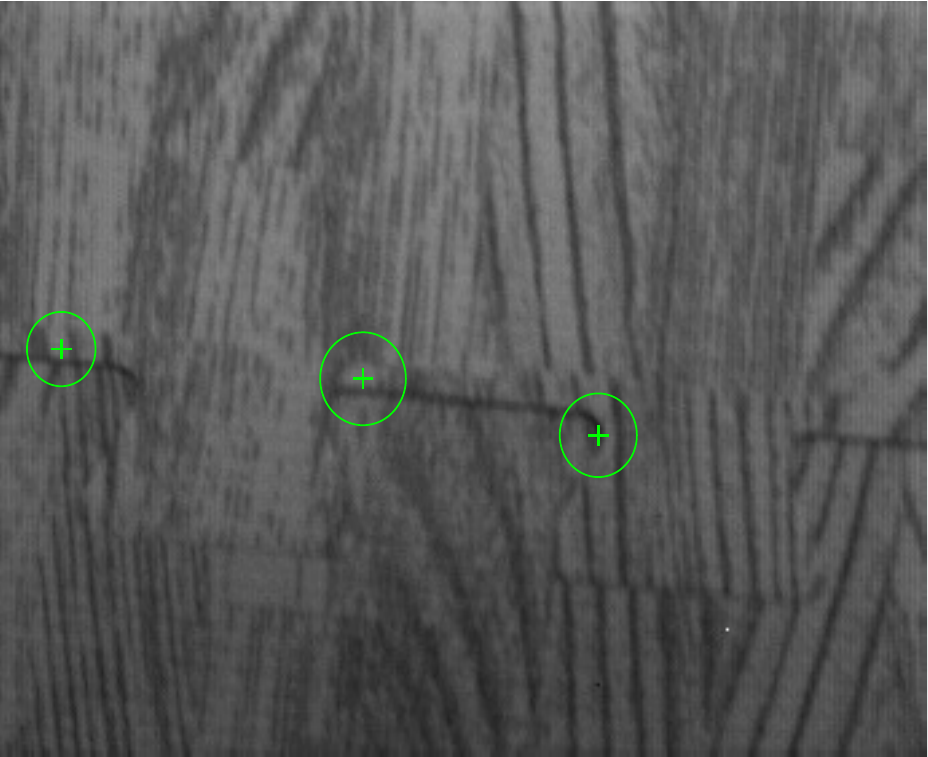} }
\subfigure[$\boldsymbol{\theta} = \mathbf{0}$]{
	\includegraphics[width=0.23\textwidth]{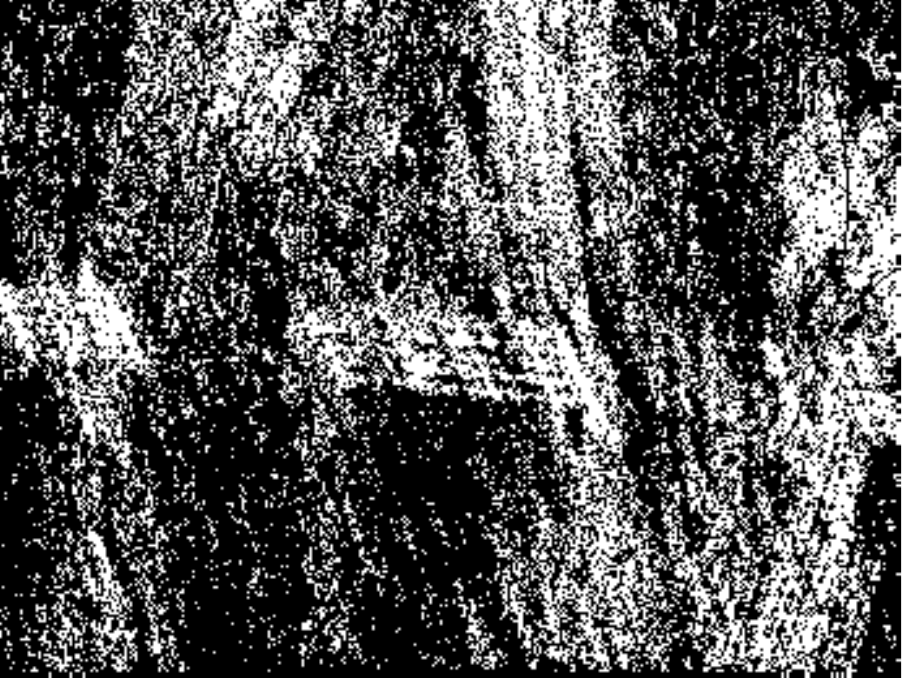}}
\subfigure[$\boldsymbol{\theta} =\hat{\boldsymbol{\theta}}$]{
	\includegraphics[width=0.23\textwidth]{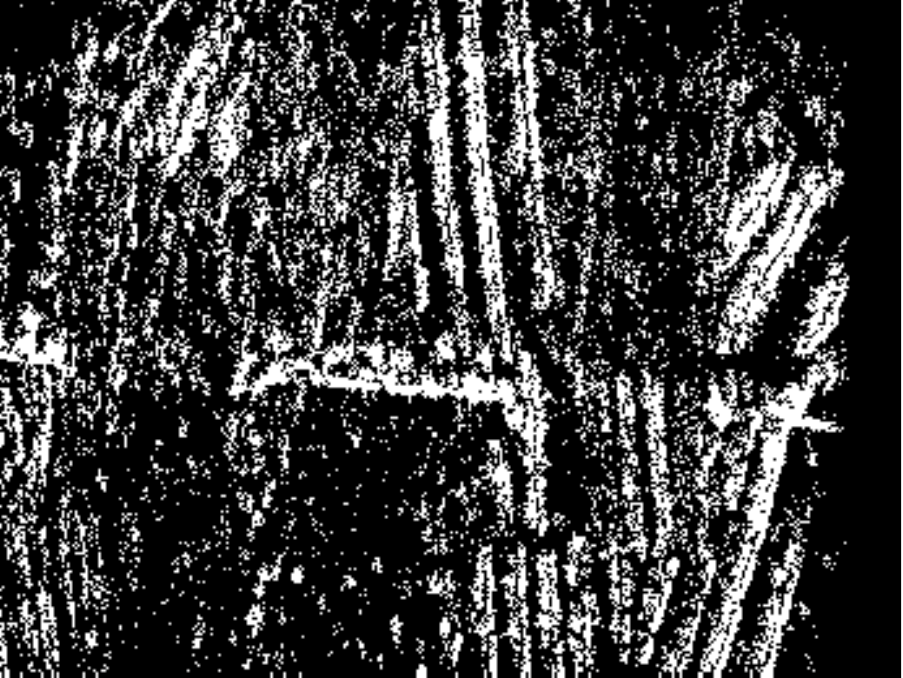}} 

\caption{(a): AGV equipped with a downward facing event camera for vehicle motion estimation. (b)-(d): collected image with detectable corners, image of warped events with $\boldsymbol{\theta} = \mathbf{0}$, and image of warped events with optimal parameters $\hat{\boldsymbol{\theta}}$.}
\label{fig:robot and frame}
\end{figure}


\section{Contrast Maximisation}

Gallego~et~al.~\cite{gallego2018unifying} recently introduced contrast maximisation as a unifying framework allowing the solution of several important problems for dynamic vision sensors, in particular motion estimation problems in which the effect of camera motion may be described by a homography (e.g. motion in front of a plane, pure rotation). Our work relies on contrast maximisation, which we therefore briefly review in the following.

An event camera outputs a sequence of \textit{events} denoting temporal logarithmic brightness changes above a certain threshold. An event $e = \left\{ \mathbf{x},\ t,\ s \right\}$ is described by its pixel position $\mathrm{\bold{x}} = [x\text{ }y]^T$, timestamp $t$, and polarity $s$ (the latter indicates whether the brightness is increasing or decreasing, and is ignored in the present work). The core idea of contrast maximisation is relatively straightforward: The flow of the events is modelled by a time-parametrised homography. Given its position and time-stamp, every event may therefore be warped back along a point-trajectory into a reference view with timestamp $t_{\mathrm{ref}}$. Since events are more likely to be generated by high-gradient edges, correct homographic warping parameters will likely lead to a sharp Image of Warped Events (IWE) in which events align along a crisp edge-map. Gallego~et~al.~\cite{gallego2018unifying} simply propose to consider the contrast of the IWE as a reward function to identify the correct homographic warping parameters. Note that homographic warping functions include 2D affine and Euclidean transformations, and thus can be used in a variety of vision problems such as optical flow, feature tracking, or fronto-parallel motion estimation.
 
Suppose we are given a set of $N$ events $\mathcal{E} = \{e_{k}\}_{k=1}^{N}$. We define a general warping function ${\mathrm{\bold{x}}_{k}^{\prime}} = W(\mathrm{\bold{x}}_k,t_{k};\boldsymbol{\theta})$ that returns the position ${\mathrm{\bold{x}}_{k}^{\prime}}$ of an event $e_k$ in the reference view at time $t_{\mathrm{ref}}$.  $\boldsymbol{\theta}$ is a vector of warping parameters. The IWE is generated by accumulating warped events at each discrete pixel location:
\begin{equation}
	I(\mathbf{p}_{ij};\boldsymbol{\theta}) = \sum_{k = 1}^{N}\mathbf{1}(\mathbf{p}_{ij}-{\mathrm{\bold{x}}_{k}^{\prime}}) =
	\sum_{k = 1}^{N}\mathbf{1}(\mathbf{p}_{ij}-W(\mathrm{\bold{x}}_k,t_{k};\boldsymbol{\theta})) \,,
	\label{equ:pixel intensity}
\end{equation}
where $\mathbf{1}(\cdot)$ is an indicator function that counts 1 if the absolute value of $(\mathbf{p}_{ij} - {\mathrm{\bold{x}}_{k}^{\prime}})$ is less than a threshold $\epsilon$ in each coordinate, and otherwise 0. $\mathbf{p}_{ij}$ is a pixel in the IWE with coordinates $[i\text{ }j]^T$, and we refer to it as an \textit{accumulator} location. We set $\epsilon = 0.5$ such that each warped event will increment one accumulator only.

Existing approaches replace the indicator function with a Gaussian kernel to make the IWE a smooth function of the warped events, and thus solve contrast maximisation problems via local optimisation methods (cf. \cite{gallego2017accurate,gallego2018unifying,gallego2019focus}).
In contrast, we show how our proposed method is able to find the global optimum of the above, discrete objective function.

As introduced in \cite{stoffregen2019event1,gallego2019focus}, reward functions for event un-warping all rely on the idea of maximising the contrast or sharpness of the IWE (they have also been denoted as \textit{focus loss functions}). They proceed by integration over the entire set of accumulators, which we denote $\mathcal{P}$. The most relevant ones for us are summarized in Table~\ref{tab:contrast_functions}. Note that for $L_{\mathrm{Var}}$, $\mu_{I}$ is the mean value of $I(\mathbf{p}_{ij};\boldsymbol{\theta})$ over all pixels (a function of $\boldsymbol{\theta}$ itself), and $N_{p}$ is the total number of accumulators in $I$. For $L_{\mathrm{SoSA}}$, $\delta$ is a design parameter called the \textit{shift factor}. Different from other objectives functions, locations with few accumulations will contribute more to $L_{\mathrm{SoSA}}$. The intuition here is that more empty locations again mean more events that are concentrated at fewer accumulators. $L_{\mathrm{SoEaS}}$ is a combination of $L_{\mathrm{SoS}}$ and $L_{\mathrm{SoE}}$. Similarly, $L_{\mathrm{SoSAaS}}$ is a combination of $L_{\mathrm{SoS}}$ and $L_{\mathrm{SoSA}}$.

\begin{table}[t]
\centering
\caption{Contrast functions evaluated in this work}
\label{tab:contrast_functions}
\renewcommand\arraystretch{1.5}
\begin{tabular}{|l|l|}
\hline
\ Sum of Squares (SoS) \ &
	\ $L_{\mathrm{SoS}}(\boldsymbol{\theta}) = \sum_{\mathbf{p}_{ij}\in\mathcal{P}}I(\mathbf{p}_{ij};\boldsymbol{\theta})^2$ \ 
\\ \hline
\ Variance (Var) \ &
	 \ $L_{\mathrm{Var}}(\boldsymbol{\theta}) = \frac{1}{N_{p}}\sum_{\mathbf{p}_{ij}\in\mathcal{P}}(I(\mathbf{p}_{ij};\boldsymbol{\theta})-\mu_{I})^2$ \ 
\\ \hline
\ Sum of Exponentials (SoE) \ &
	 \ $L_{\mathrm{SoE}}(\boldsymbol{\theta}) = \sum_{\mathbf{p}_{ij}\in\mathcal{P}}e^{I(\mathbf{p}_{ij};\boldsymbol{\theta})}$ \ 
\\ \hline
\ Sum of Suppressed \ &
	 \ $L_{\mathrm{SoSA}}(\boldsymbol{\theta}) = \sum_{\mathbf{p}_{ij}\in\mathcal{P}}e^{-I(\mathbf{p}_{ij};\boldsymbol{\theta})\cdot\delta}$ \ 
\\ 
\ Accumulations (SoSA) \ & \\
\hline
\ SoE and Squares (SoEaS) \ &
	 \ $L_{\mathrm{SoEaS}}(\boldsymbol{\theta}) = \sum_{\mathbf{p}_{ij}\in\mathcal{P}}I(\mathbf{p}_{ij};\boldsymbol{\theta})^2+e^{I(\mathbf{p}_{ij};\boldsymbol{\theta})}$ \ 
\\ \hline
\ SoSA and Squares (SoSAaS) \ &
	 \ $L_{\mathrm{SoSAaS}}(\boldsymbol{\theta}) = \sum_{\mathbf{p}_{ij}\in\mathcal{P}}I(\mathbf{p}_{ij};\boldsymbol{\theta})^2+e^{-I(\mathbf{p}_{ij};\boldsymbol{\theta})\cdot\delta}$ \ 
\\ \hline
\end{tabular}
\end{table}

Let us now proceed to the main contribution of our work, which is a derivation of bounds on the above objectives as required by Branch and Bound.


\section{Globally Maximised Contrast using Branch and Bound}

Figure~\ref{fig:heat map} illustrates how contrast maximisation for motion estimation is in general a non-convex problem, meaning that local optimisation may be sensitive to the initial parameters and not find the global optimum. We tackle this problem by introducing a globally optimal solution to contrast maximisation using Branch and Bound (BnB) optimisation. BnB is an algorithmic paradigm in which the solution space is subdivided into branches in which we then find upper and lower bounds for the maximal objective value. The globally optimal solution is isolated by an iterative search in which entire branches are discarded if their upper bound for the maximum objective value remains lower than the corresponding lower bound in another branch. The most important factor deciding the effectiveness of this approach is given by the tightness of the bounds. 

Our core contribution is given by a recursive method to efficiently calculate upper and lower bounds for the maximum value of a contrast maximisation function over a given branch. In short, the main idea is given by expressing a bound over $(N+1)$ events as a function of the bound over $N$ events plus the contribution of one additional event. The strategy can be similarly applied to all six aforementioned contrast functions, which is why we limit the exposition to the derivation of bounds for $L_{\mathrm{SoS}}$. Detailed derivations for all loss functions are provided in the supplementary material.

\begin{figure}[t]
\centering
\subfigure[N/E = 0]{
\includegraphics[width=0.22\textwidth]{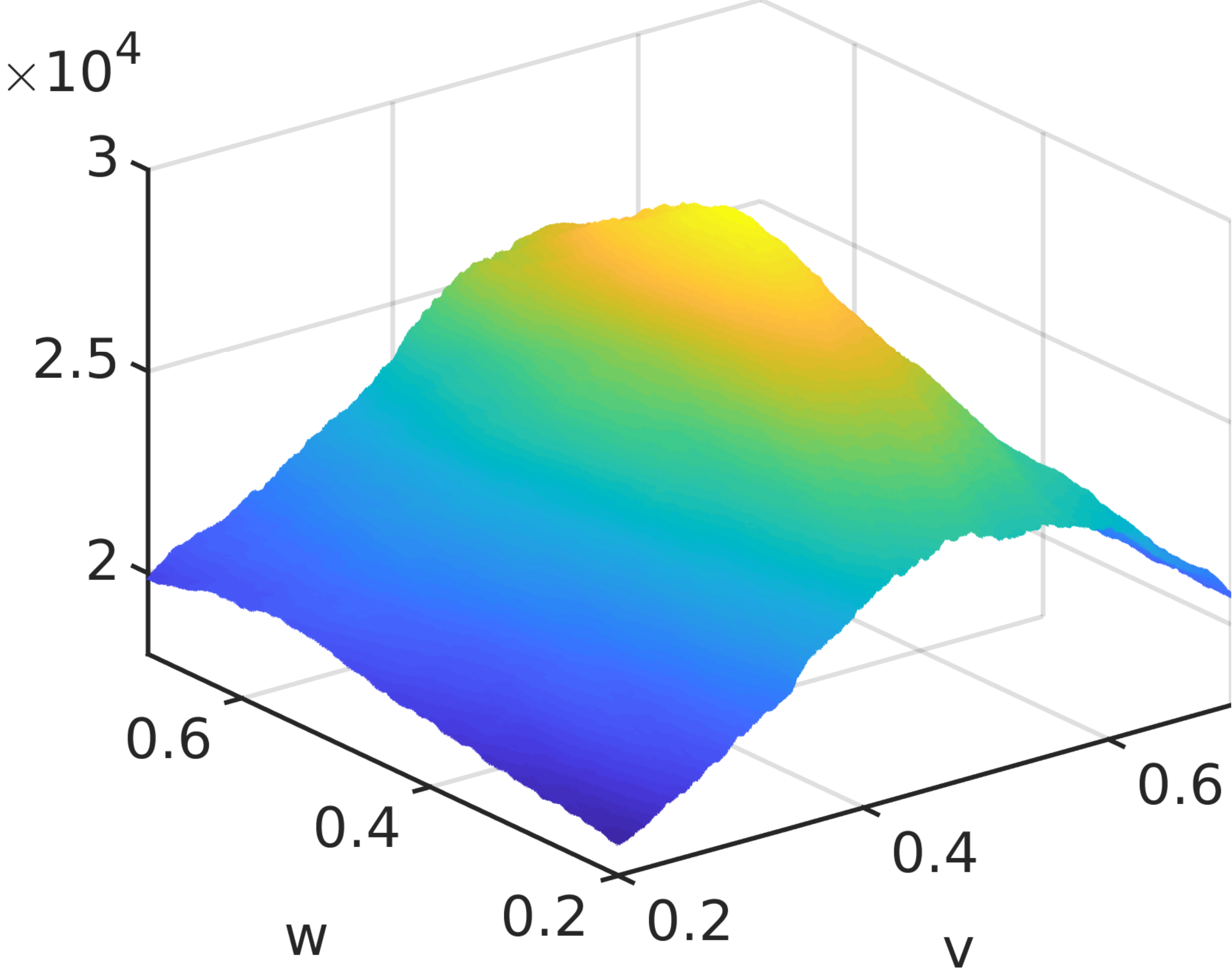} 
\label{fig:NtoE_0}
}
\subfigure[N/E = 0.02]{
\includegraphics[width=0.22\textwidth]{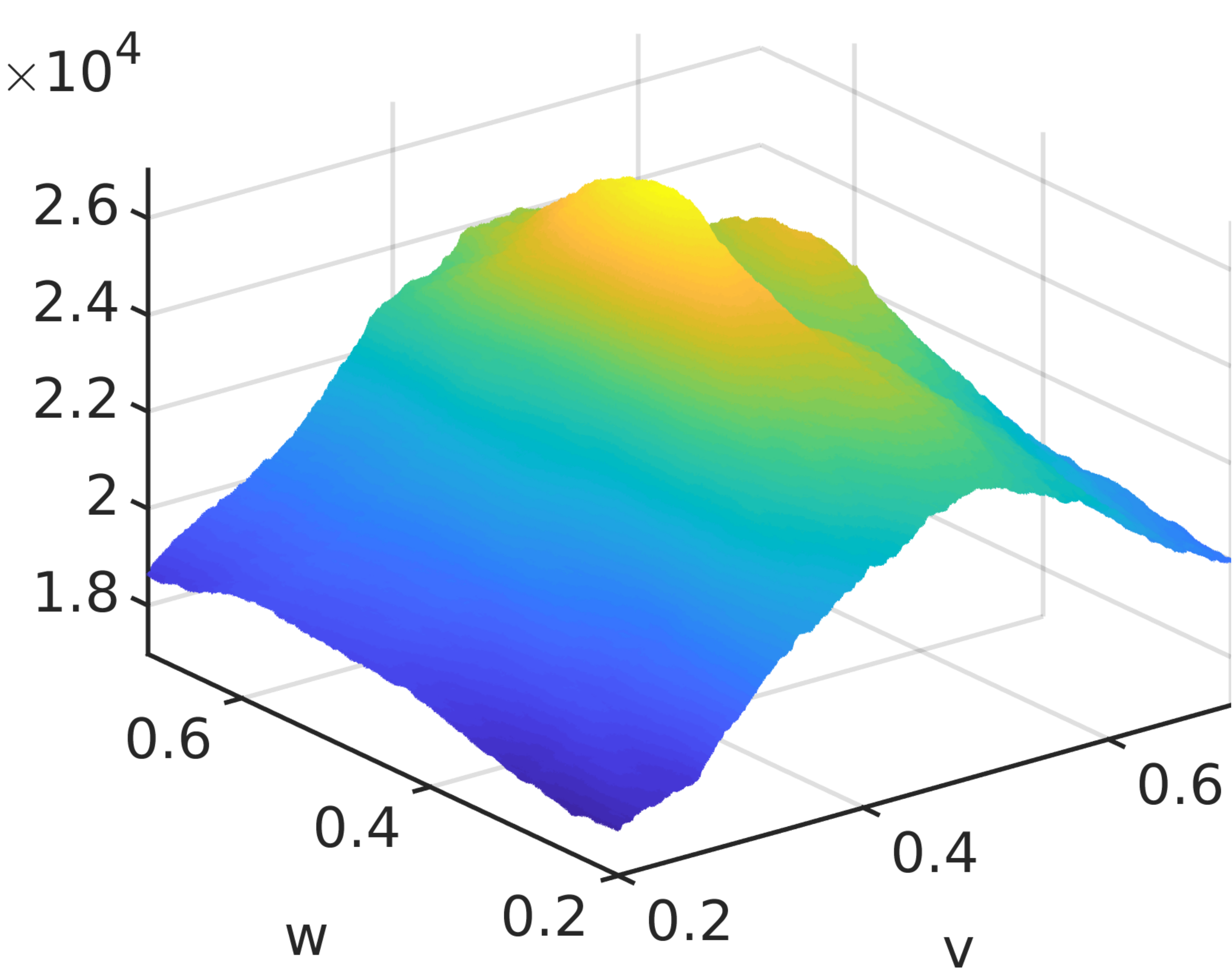}
\label{fig:NtoE_002}
}
\subfigure[N/E = 0.10]{
\includegraphics[width=0.22\textwidth]{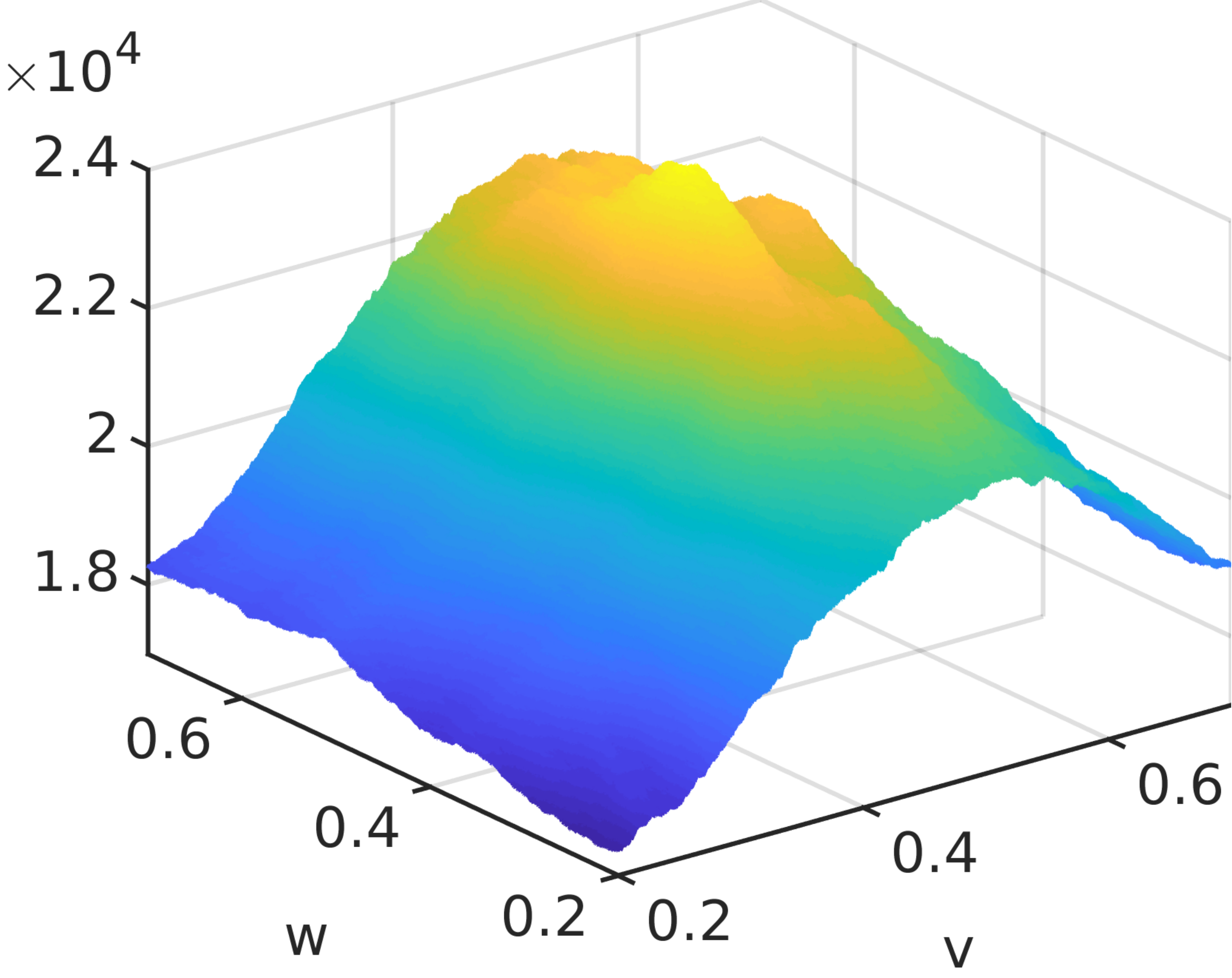}
\label{fig:NtoE_01}
}
\subfigure[N/E = 0.18]{
\includegraphics[width=0.22\textwidth]{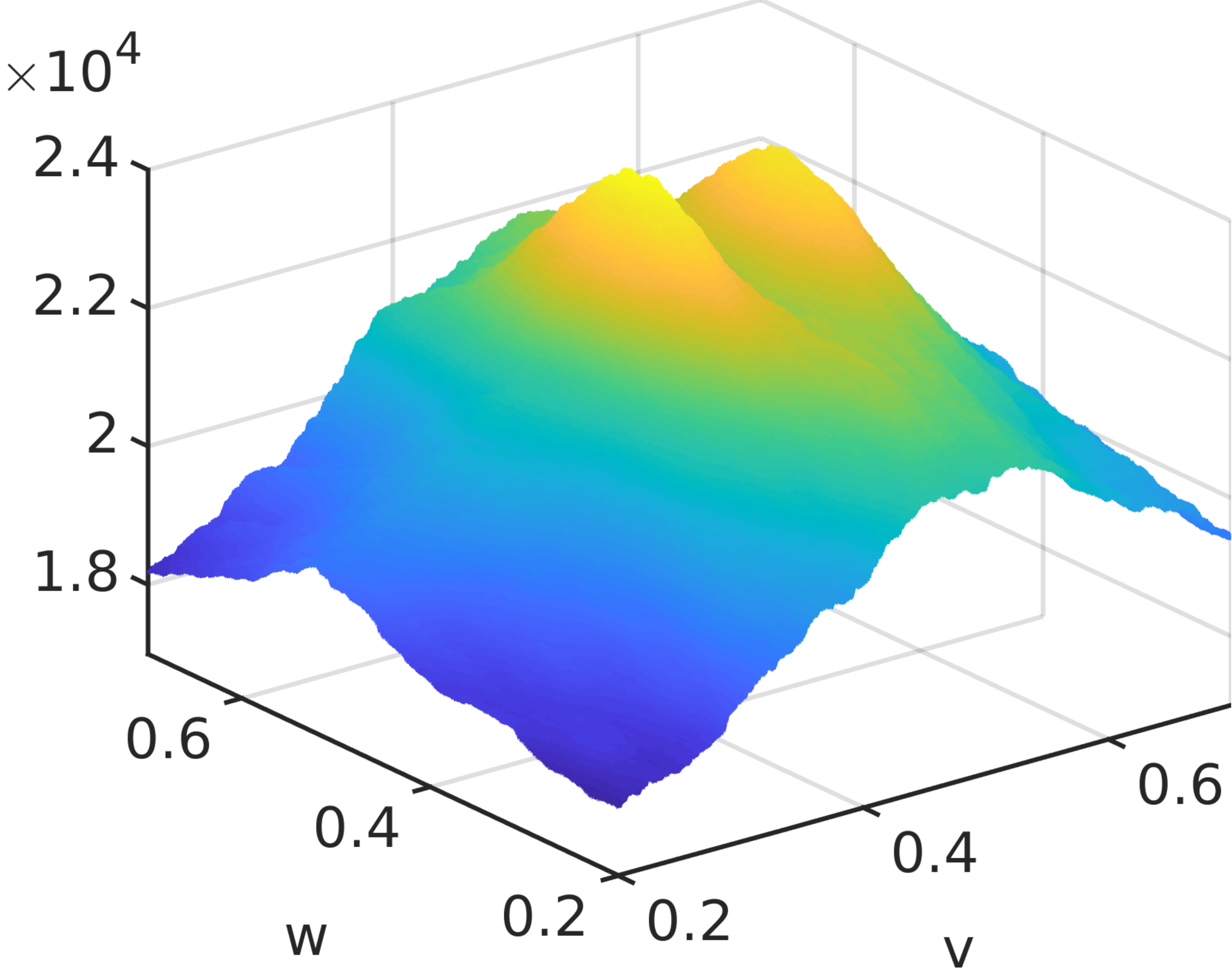} 
\label{fig:NtoE_018}
}
\caption{Visualization of the Sum of Squares contrast function. The camera is moving in front of a plane, and the motion parameters are given by translational and rotational velocity (cf. Section \ref{sec:casestudy}). The sub-figures from left to right are functions with increasing Noise-to-Events (N/E) ratios. Note that contrast functions are non-convex.}
\label{fig:heat map}
\end{figure}

\subsection{Objective Function}

In the following, we assume that $L=L_{\mathrm{SoS}}$. The maximum objective function value over all $N$ events in a given time interval $[t_{\mathrm{ref}}, t_{\mathrm{ref}}+\Delta T]$ is given by
\begin{equation}
	L_{N} = \max_{\substack{\boldsymbol{\theta}\in\boldsymbol{\Theta}}} \sum_{\textbf{p}_{ij}\in\mathcal{P}} \left[ \sum_{k = 1}^{N}\boldsymbol{1} \left( \textbf{p}_{ij}-W(\mathrm{\bold{x}}_k,t_{k};\boldsymbol{\theta}) \right) \right]^2 \,,
\end{equation}
where $\boldsymbol{\Theta}$ is the search space (i.e. branch or sub-branch) over which we want to maximise the objective. Most globally optimal methods for geometric computer vision problems find bounds by a spatial division of the problem into individual, simpler maximisation sub-problems  (cf. \cite{campbell2017globally}). However, the contrast maximisation objective is related to the distribution over the entire IWE and not just individual accumulators, which complicates this strategy. 

\subsection{Upper and Lower Bound}

The bounds are calculated recursively by processing the events and one-by-one, each time updating the IWE. The event are notably processed in temporal order with increasing timestamps.

For the lower bound, it is readily given by evaluating the contrast function at an arbitrary point on the interval $\boldsymbol{\Theta}$, which is commonly picked as the interval center $\boldsymbol{\theta}_{0}$. We present a recursive rule to efficiently evaluate the lower bound.

\noindent \textbf{Theorem 1.} \textit{For search space $\boldsymbol{\Theta}$ centered at $\boldsymbol{\theta}_{0}$, the lower bound of $SoS$-based contrast maximisation may be given by}
\begin{equation}
	\underline{L_{N+1}} = \underline{L_{N}}+1+2 I^{N}(\boldsymbol{\eta}_{N+1}^{\theta_0};\boldsymbol{\theta}_{0}) \,,
\label{equ:lower bound}
\end{equation}
\textit{where $I^N(\mathbf{p}_{ij};\boldsymbol{\theta}_{0})$ is the incrementally constructed IWE, its exponent $N$ denotes the number of events that have already been taken into account, and}
\begin{equation}
	\boldsymbol{\eta}_{N+1}^{\theta_0} = \operatorname{round}(W(\mathrm{\bold{x}}_{N+1},t_{N+1};\boldsymbol{\theta}_{0}))
\end{equation}
\textit{returns the accumulator closest to the warped position of the $(N+1)$-th event.}
\begin{proof}
    According to the definition of sum of the square focus loss function, 
\begin{equation}
\begin{split}
	\underline{L_{N+1}} &=  \sum_{\textbf{p}_{ij}\in\mathcal{P}} \left[ \sum_{k = 1}^{N+1}\boldsymbol{1} \left( \textbf{p}_{ij}-W(\mathrm{\bold{x}}_k,t_{k};\boldsymbol{\theta_0}) \right) \right]^2 \\
	&= \sum_{\mathbf{p}_{ij} \in \mathcal{P}} 
    		   \left[ I^{N}(\mathbf{p}_{ij};\boldsymbol{\theta}_0)
    		          + \boldsymbol{1} \left( \mathbf{p}_{ij}-W(\mathrm{\bold{x}}_{N+1},t_{N+1};\boldsymbol{\theta_0}) \right) \right]^2 \\
	        &= a+b+c \, \text{, where}
\end{split}
\end{equation}
\begin{equation*}
\begin{split}
    & a =  \sum_{\mathbf{p}_{ij} \in \mathcal{P}} I^{N}(\mathbf{p}_{ij};\boldsymbol{\theta}_0)^2 \,, \\ 
    & b =  2\sum_{\mathbf{p}_{ij}\in\mathcal{P}}
    	  \left[ \boldsymbol{1}(\mathbf{p}_{ij}-W(\mathrm{\bold{x}}_{N+1},t_{N+1};\boldsymbol{\theta}_0)) I^{N}(\mathbf{p}_{ij};\boldsymbol{\theta}_0) \right] \,, \\
    & c =  \sum_{\mathbf{p}_{ij}\in\mathcal{P}}
    	  \left[ \boldsymbol{1}(\mathbf{p}_{ij}-W(\mathrm{\bold{x}}_{N+1},t_{N+1};\boldsymbol{\theta}_0)) \right]^2
    	  \,.
\end{split}
\end{equation*}
It is clear that $a = \underline{L_{N}}$. In $c$, owing to the definition of our indicator function, only the $\mathbf{p}_{ij}$ which is closest to $W(\mathrm{\bold{x}}_{N+1},t_{N+1};\boldsymbol{\theta}_0)$ makes a contribution, thus we have $c = 1$. For $b$, the term $\boldsymbol{1}(\mathbf{p}_{ij}-W(\mathrm{\bold{x}}_{N+1},t_{N+1};\boldsymbol{\theta}_0))$ is simply zero unless we are considering an accumulator $\mathbf{p}_{ij} = \boldsymbol{\eta}_{N+1}^{\theta_0}$, which gives $b = 2I^{N}(\boldsymbol{\eta}_{N+1}^{\theta_0};\boldsymbol{\theta}_{0})$. Thus we obtain (\ref{equ:lower bound}). Note that the IWE is iteratively updated by incrementing the accumulator which locates closest to $ \boldsymbol{\eta}_{N+1}^{\theta_0}$.
\end{proof}

\begin{figure*}[t]
    \centering
    \subfigure[]
    {
    \includegraphics[width=0.68\textwidth]{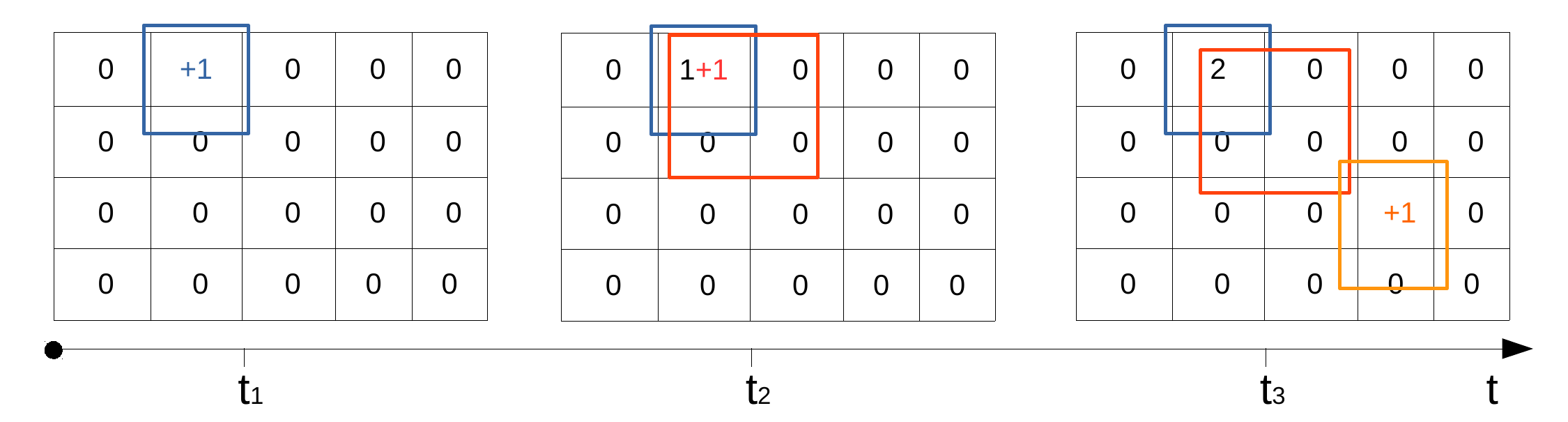}
    \label{fig:iweConstruction}
    }
    \subfigure[]
    {
    \includegraphics[width = 0.27\textwidth, height = 0.2\textwidth]{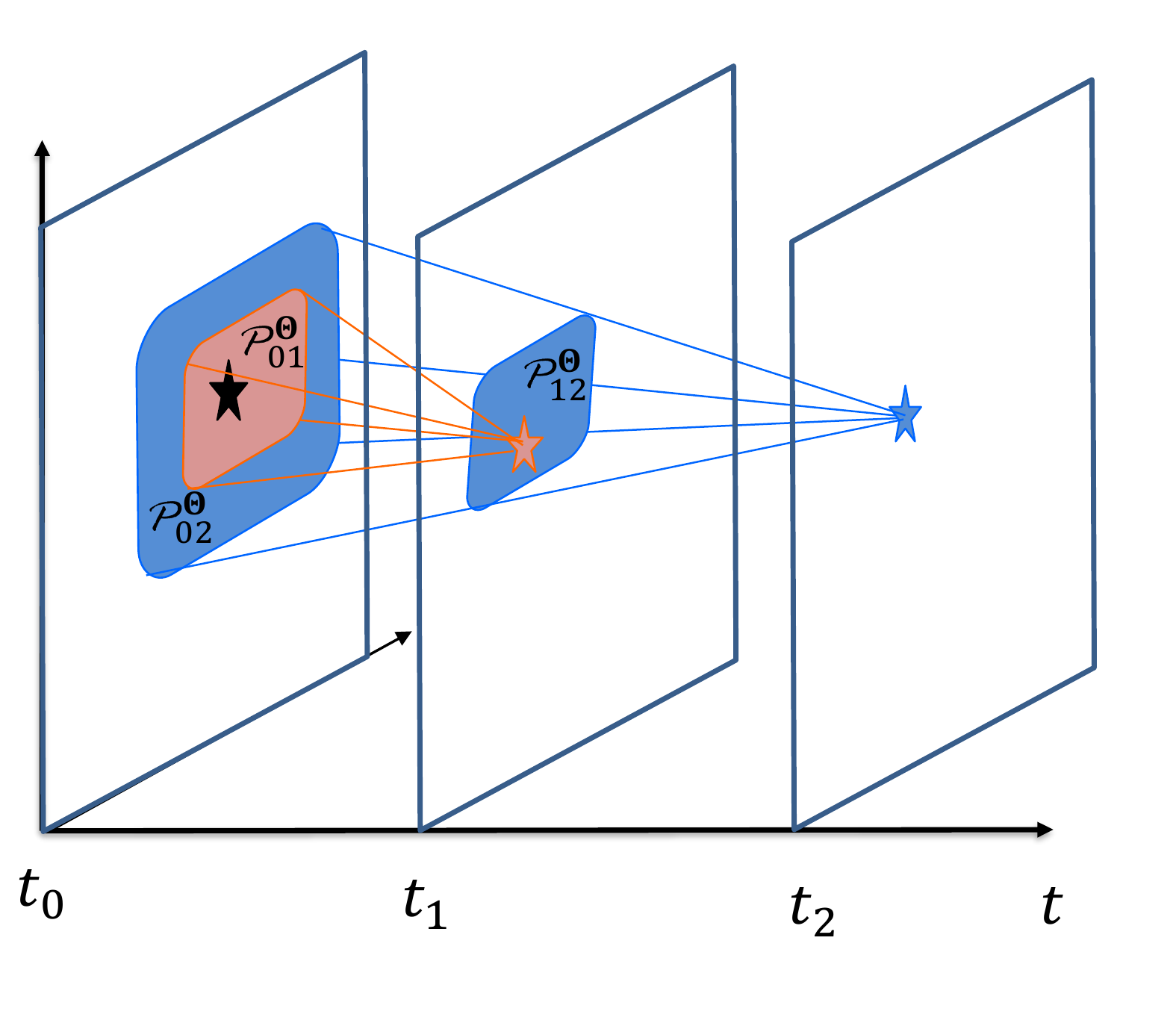}
    \label{fig:nested_boundingbox}
    }
    \caption{(a) Incremental update of the IWE. For each new event $e$, we choose and increment the currently maximal accumulator in the bounding box $\mathcal{P}^{\boldsymbol{\Theta}}$ around all possible locations $W(\mathrm{\bold{x}},t;\boldsymbol{\theta}\in\boldsymbol{\Theta})$. We simply increment the center of the bounding box if no other accumulator exists. (b) Bounding boxes of two temporally distinct events generated by the same point in 3D.}
\end{figure*}

We now proceed to our main contribution, a recursive upper bound for the contrast maximisation problem. Let us define $\mathcal{P}^{\boldsymbol{\Theta}}_{i}$ as the bounding box around all possible locations $W(\mathrm{\bold{x}}_i,t_i;\boldsymbol{\theta}\in\boldsymbol{\Theta})$ of the un-warped event. Lemma~\ref{lemma:nest_boundingbox} is introduced as follows.

\begin{lemma}
Given a search space $\boldsymbol{\theta} \in \boldsymbol{\Theta}$, for a small enough time interval, if $W(\mathrm{\bold{x}}_i,t_i;\boldsymbol{\theta})$ = $W(\mathrm{\bold{x}}_j,t_j;\boldsymbol{\theta})$ and $0 < i < j \leq N$, we have $\mathcal{P}^{\boldsymbol{\Theta}}_{i} \subseteq \mathcal{P}^{\boldsymbol{\Theta}}_{j}$. An intuitive explanation is given in Figure \ref{fig:nested_boundingbox}.
\label{lemma:nest_boundingbox}
\end{lemma}

Lemma~\ref{lemma:nest_boundingbox} now enables us to derive our recursive upper bound.

\begin{theorem}
The upper bound of the objective function $L_N$ for SoS-based contrast maximisation satisfies
\begin{eqnarray}
     L_{N+1} &=& L_{N}+1+2I^{N}(\boldsymbol{\eta}_{N+1}^{\hat{\theta}};\hat{\boldsymbol{\theta}}) \label{eq:opt_recursive_formulation}\\
    & \leq& \overline{L_{N}}+1+2Q^N = \overline{L_{N+1}},  \  \label{eq:upper_bound} \\  
   \text{where } Q^N  & =& \max_{\substack{\mathbf{p}_{ij}\in\mathcal{P}_{N+1}^{\boldsymbol{\Theta}}}} \overline{I}^{N}(\mathbf{p}_{ij}) \geq I^{N}(\boldsymbol{\eta}_{N+1}^{\hat{\theta}};\hat{\boldsymbol{\theta}}) \nonumber \,
\end{eqnarray}
$\mathcal{P}_{N+1}^{\boldsymbol{\Theta}}$ is a bounding box for the $(N+1)$-th event. $\hat{\boldsymbol{\theta}}$ is the optimal parameter set that maximises $L_{N+1}$ over  the interval $\boldsymbol{\Theta}$. $\overline{I}^{N}(\mathbf{p}_{ij})$ is the value of pixel $\mathbf{p}_{ij}$ in the upper bound IWE, a recursively constructed image in which we always increment the maximum accumulator within the bounding box $\mathcal{P}_{N+1}^{\boldsymbol{\Theta}}$ (i.e. the one that we used to define the value of $Q^N$. The incremental construction of $\overline{I}^{N}(\mathbf{p}_{ij})$ is illustrated in Figure~\ref{fig:iweConstruction}.
\label{the:upper bound}
\end{theorem}
\begin{proof}
(\ref{eq:opt_recursive_formulation}) is straightforwardly derived from (\ref{equ:lower bound}). The proof of inequation (\ref{eq:upper_bound}) then proceeds by mathematical induction. \\

For $N$ = 0, it is obvious that $L_{0} = \overline{L_{0}}= 0$.  Similarly, for $N$ = 1, $L_{1} = 1 \leq \overline{L_{0}} + 1 + 0$, and $Q^0 = I^{0}(\boldsymbol{\eta}_{1}^{\hat{\theta}};\hat{\boldsymbol{\theta}}) = 0$ (which satisfies Theorem~\ref{the:upper bound}).  We now assume that $\overline{L_n}$ as well as the corresponding upper bound IWE $\overline{I}^{n}$ are given for all $0<n\leq N$. We furthermore assume that they satisfy Theorem~\ref{the:upper bound}. Our aim is to prove that \eqref{eq:upper_bound} holds for the $(N+1)$-th event. It is clear that $\overline{L_N} \geq L_N$, and we only need to prove that $Q^N \geq I^{N}(\boldsymbol{\eta}_{N+1}^{\hat{\theta}};\hat{\boldsymbol{\theta}})$, for which we will make use of Lemma 1. There are two cases to be distinguished:
\begin{itemize}
    \item The first case is if there exists an event $\epsilon_k$ with $0 < k < N+1$ and for which $\boldsymbol{\eta}_{k}^{\hat{\theta}} = \boldsymbol{\eta}_{N+1}^{\hat{\theta}}$. In other words, the $k$-th and the $(N+1)$-th event are warped to a same accumulator if choosing the locally optimal parameters. Note that if there are multiple previous events for which this condition holds, the $k$-th event is chosen to be the most recent one. Given our assumptions, $\overline{L_{k-1}}$ as well as the $(k-1)$-th constructed upper bound IWE satisfy Theorem~\ref{the:upper bound}, which means that $Q^{k-1} \geq I^{k-1}(\boldsymbol{\eta}_{k}^{\hat{\theta}};\hat{\boldsymbol{\theta}})$. Let $\mathbf{p}_k \in \mathcal{P}^{\Theta}_k$ now be the pixel location with maximum intensity in $\overline{I}^{k-1}(\mathbf{p}_k)$. Then, the $k$-th updated IWE satisfies  $\overline{I}^k(\mathbf{p}_k) = Q^{k-1}+1 \geq I^{k-1}(\boldsymbol{\eta}_{k}^{\hat{\theta}};\hat{\boldsymbol{\theta}}) +1 $. According to Lemma~\ref{lemma:nest_boundingbox}, we have $\mathcal{P}^{\Theta}_k \subseteq \mathcal{P}^{\Theta}_{N+1}$, therefore $\mathbf{p}_k \subseteq \mathcal{P}^{\Theta}_{N+1}$, and $Q^N \geq \overline{I}^k(\mathbf{p}_k) \geq  I^{k-1}(\boldsymbol{\eta}_{k}^{\hat{\theta}};\hat{\boldsymbol{\theta}})+1 $. With optimal warp parameters $\hat{\boldsymbol{\theta}}$, events with indices from $k+1$ to $N$ will not locate at $\boldsymbol{\eta}_{N+1}^{\hat{\theta}} $
    , and therefore $I^{k-1}(\boldsymbol{\eta}_{k}^{\hat{\theta}};\hat{\boldsymbol{\theta}})+1 = I^{N}(\boldsymbol{\eta}_{N+1}^{\hat{\theta}};\hat{\boldsymbol{\theta}}) \leq Q^N $.
    \item If there is no such a event, it is obvious that $Q^N \geq  I^{N}(\boldsymbol{\eta}_{N+1}^{\hat{\theta}};\hat{\boldsymbol{\theta}})$.
    \end{itemize}
With the basic cases and the induction step proven, we conclude our proof that Theorem~\ref{the:upper bound} holds for all natural numbers $N$. 
\end{proof}

We apply the proposed strategy to derive upper and lower bounds for all six aforementioned contrast functions, and list them in Table~\ref{tab:upper_bounds}. Note that the initial case varies for different loss functions. The globally-optimal contrast maximisation framework~(GOCMF) is outlined in Algorithm~\ref{alg:GOCMF} and Algorithm~\ref{alg:RB}. We propose a nested strategy for calculating upper bounds, in which the outer layer $RB$ evaluates the objective function, while the inner layer $BB$ estimates the bounding box $\mathcal{P}_{N}^{\boldsymbol{\Theta}}$ and depends on the specific motion parametrisation.
\begin{table}[t]
\centering
\caption{Recursive Upper and Lower Bounds}
\label{tab:upper_bounds}
\renewcommand\arraystretch{1.5}
\begin{tabular}{|c|c|c|c|}
\hline
 & \textbf{Upper Bound} $\overline{L_{N}}$ & \textbf{Lower Bound} $\underline{L_{N}}$ & $L_0$ \\ \hline
\tiny \textbf{SoS}       & \tiny $ \overline{L_{N-1}} + 1 + 2 Q$                    & \tiny  $\underline{L_{N-1}} + 1 + 2 I^{N-1}(\boldsymbol{\eta}_N^{\theta_0};\boldsymbol{\theta}_0)$  &   \tiny 0                     \\ \hline
\tiny \textbf{Var}       & \tiny  $\overline{L_{N-1}} + \frac{1}{N_{p}} - \frac{2 \mu_{I}}{N_{p}} + \frac{2}{N_{p}} Q$   & \tiny  $\underline{L_{N-1}} + \frac{1}{N_{p}} - \frac{2 \mu_{I}}{N_{p}} + \frac{2}{N_{p}} I^{N-1}(\boldsymbol{\eta}_N^{\theta_0};\boldsymbol{\theta}_0)$                    &  \tiny $\mu_{I}^2$                     \\ \hline
\tiny \textbf{SoE}       &  \tiny $\overline{L_{N-1}} + (e-1) e^{Q}$                   &  \tiny $\underline{L_{N-1}} + (e-1) e^{I^{N-1}(\boldsymbol{\eta}_N^{\theta_0};\boldsymbol{\theta}_0)}$                    &   \tiny $N_{p}$                    \\ \hline
\tiny \textbf{SoSA}      &  \tiny $\overline{L_{N-1}} + (e^{-\delta}-1) e^{-\delta \cdot Q}$                   &  \tiny $\underline{L_{N-1}} + (e^{-\delta}-1) e^{-\delta \cdot I^{N-1}(\boldsymbol{\eta}_N^{\theta_0};\boldsymbol{\theta}_0)}$                    &   \tiny $N_{p}$                    \\ \hline
\tiny \textbf{SoEaS}     & \tiny $\overline{L_{N-1}} + 1 + 2 Q + (e-1) e^{Q}$            & \tiny $\underline{L_{N-1}} + 1 + 2 I^{N-1}(\boldsymbol{\eta}_N^{\theta_0};\boldsymbol{\theta}_0) + (e-1) e^{I^{N-1}(\boldsymbol{\eta}_N^{\theta_0};\boldsymbol{\theta}_0)}$            &   \tiny $N_{p}$                                \\ \hline
\tiny \textbf{SoSAaS}    &  \tiny $\overline{L_{N-1}} + 1 + 2Q + (e^{-\delta}-1)e^{-\delta Q}$   &  \tiny $\underline{L_{N-1}} + 1 + 2I^{N-1}(\boldsymbol{\eta}_N^{\theta_0};\boldsymbol{\theta}_0) + (e^{-\delta}-1)e^{-\delta I^{N-1}(\boldsymbol{\eta}_N^{\theta_0};\boldsymbol{\theta}_0)}$                   &                      \tiny $N_{p}$                                \\ \hline
\end{tabular}
\end{table}
\begin{minipage}[t]{0.5\textwidth}
    \begin{algorithm}[H]
    	\caption{GOCMF: globally optimal contrast maximisation framework} 
 		{\bf Input:} 
		event set $\mathcal{E}$, initial search space $\boldsymbol{\Theta}$, branching limit $N_b$\\
 		{\bf Output:} 
 		optimal warping parameters $\hat{\boldsymbol{\theta}}$
 		\begin{algorithmic}[1]
 			\State Initialize $\hat{\boldsymbol{\theta}}$ with the center of $\boldsymbol{\Theta}$,\\$L^* \leftarrow 0$ , $S$ $\leftarrow$ \{RB($\mathcal{E}$, $\boldsymbol{\Theta}$), $\boldsymbol{\Theta}$\}
 			\State Push $S$ into queue $Q$, $S^* \leftarrow S$
 			\While{$i < N_b$} 
 				\State $L^* \leftarrow 0$
		                 \If{$S^*.\underline{L},  == S^*.\overline{L}$} 
		                 \State $\hat{\boldsymbol{\theta}}  \leftarrow $ Center of $S^*.\boldsymbol{\Theta}$, break
		                 \EndIf
 				\For{each node $S \in Q$}
 					\State Pop $S$, split into subspaces $S_j$
 					\For{all subspaces $S_j$}
 						\State $\{S_j.\underline{L},S_j.\overline{L}\} \leftarrow $ RB($\mathcal{E}$, $\boldsymbol{\Theta}_j$)
 						\If {$S_j.\underline{L} > L^*$} 
 						\State $L^* \leftarrow S_j.\underline{L}$ , $S^* \leftarrow S_j$
 						\EndIf
 						\State Push $S_j$ into $Q$
 					\EndFor
 				\EndFor
 				\State Prune branches in $Q$

 				\State $i \leftarrow i + 1$
 			\EndWhile
 			\State \Return $\hat{\boldsymbol{\theta}}$
 		\end{algorithmic}
 		\label{alg:GOCMF}
     \end{algorithm}
\end{minipage}
\hspace{0.02cm}
\begin{minipage}[t]{0.47\textwidth}
    \begin{algorithm}[H]
		\caption{RB: recursive bounds calculation} 
		{\bf Input:} 
 		event set $\mathcal{E}$, search space $\boldsymbol{\Theta}$\\
		{\bf Output:} 
		lower bound $\underline{L}$, upper bound $\overline{L}$
		\begin{algorithmic}[1]
			\State Initialize accumulator images $\overline{I}$ and $I$ with zeros 
			\State Initialize $\underline{L}$, $\overline{L}$ according to Table~\ref{tab:upper_bounds}
			\State $\boldsymbol{\theta}_0 \leftarrow$ center of $\boldsymbol{\Theta}$ 
			\For{each event $e_k \in \mathcal{E}$}
				\State $\mathcal{P}_k^{\boldsymbol{\Theta}} \leftarrow BB(W(\cdot),\boldsymbol{\Theta},e_k) $
				\State $Q \leftarrow \max_{\substack{
    \mathbf{p}_{ij} \in \mathcal{P}_{k}^{\boldsymbol{\Theta}}
    }} \overline{I}(\mathbf{p}_{ij})$
				\State $\boldsymbol{\eta}_k^{\theta_0} \leftarrow \operatorname{round}(W(\mathrm{\bold{x}}_{k},t_{k};\boldsymbol{\theta}_{0}))$
				\State Update $\underline{L}$, $\overline{L}$ (cf. Table~\ref{tab:upper_bounds})
				\State $\boldsymbol{\nu}_k \leftarrow  \operatorname{argmax}_{\substack{
    \mathbf{p}_{ij} \in \mathcal{P}_{k}^{\boldsymbol{\Theta}}
    }} \overline{I}(\mathbf{p}_{ij})$
				\State $\overline{I}(\boldsymbol{\nu}_k) \leftarrow  \overline{I}(\boldsymbol{\nu}_k)+ 1$
				\State $I(\boldsymbol{\eta}_k^{\theta_0}) \leftarrow  I(\boldsymbol{\eta}_k^{\theta_0})+ 1$ 
			\EndFor
			\State \Return $\underline{L}$, $\overline{L}$
		\end{algorithmic}
	\label{alg:RB}
    \end{algorithm}
\end{minipage}


\section{Application to Visual Odometry with a downward-facing Event Camera}
\label{sec:casestudy}

Motion estimation for planar Autonomous Ground Vehicles~(AGVs) is an important problem in intelligent transportation~\cite{yifu2020vehicle,peng2019articulated,Huang2019vehicle}. An interesting alternative is given by employing a downward instead of a forward facing camera, thus permitting direct observation of the ground plane with known depth. This largely simplifies the geometry of the problem and notably turns the image-to-image warping into a homographic mapping that is linear in homogeneous space. The strategy is widely used in relevant applications such as sweeping robots and factory AGVs, and a good review is presented in~\cite{aqel2016review}. However, the method is affected by potentially severe challenges given by the image appearance: a) reliable feature matching or even extraction may be difficult for certain noisy ground textures, b) fast motion may easily lead to motion blur, and c) stable appearance may require artificial illumination. Many existing methods therefore do not employ feature correspondences but aim at a correspondence-less alignment or even a full photometric image alignment. Besides more classical RANSAC-based hypothesise-and-test schemes~\cite{chen2018streetmap}, the community therefore has also developed appearance-based template matching approaches~\cite{dille2010outdoor,nourani2009practical,yu2011appearance,nourani2011correlation,gonzalez2012combined}, solvers based on efficient second-order minimisation~\cite{lovegrove2011accurate,zienkiewicz2015extrinsics,jordan2016ground}, and methods exploiting the Fast Fourier Transform~\cite{piyathilaka2010experimental,birem2018visual}, the Fourier-Mellin Transform~\cite{guo2005application,kazik2011visual}, or the Improved Fourier Mellin Invariant~\cite{xu2019improved,bulow2009fast}. In an attempt to tackle highly self-similar ground textures, Dille~et~al.~\cite{dille2010outdoor} propose to use an optical flow sensor instead of a regular CMOS camera.

A critical question is given by the position of the camera. The camera may hang in the front or rear of the vehicle, which gives increased distance to the ground plane and in turn reduces motion blur. However, it also causes moving shadows in the image, and generally complicates the stabilisation of the image appearance and thus repeatable feature detection or region-based matching. A common alternative therefore is given by installing the camera underneath the vehicle paired with an artificial light source~(e.g.~\cite{dille2010outdoor,birem2018visual}). However, the short distance to the ground plane may easily lead to unwanted motion blur. We therefore consider an event camera as a highly interesting and much more dynamic alternative visual sensor for this particular scenario.

\begin{figure}[t]
  \centering
  \includegraphics[width = 0.45\textwidth]{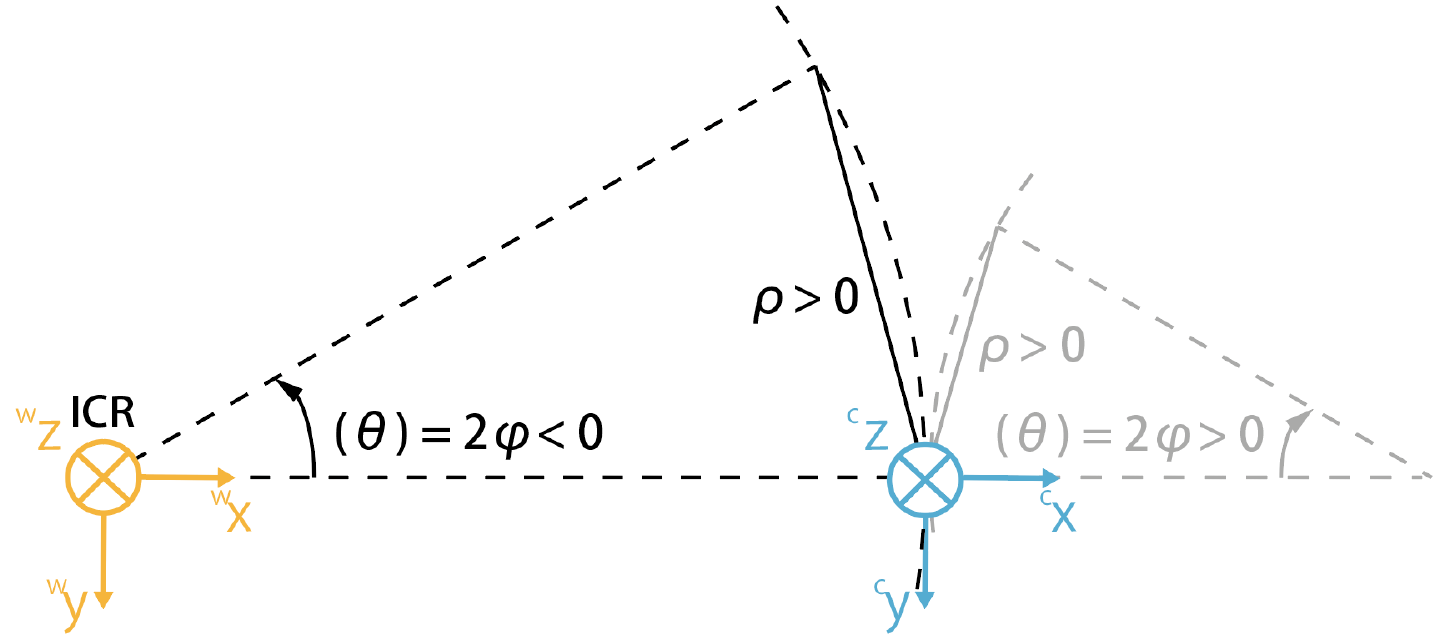}
  \includegraphics[width = 0.45\textwidth]{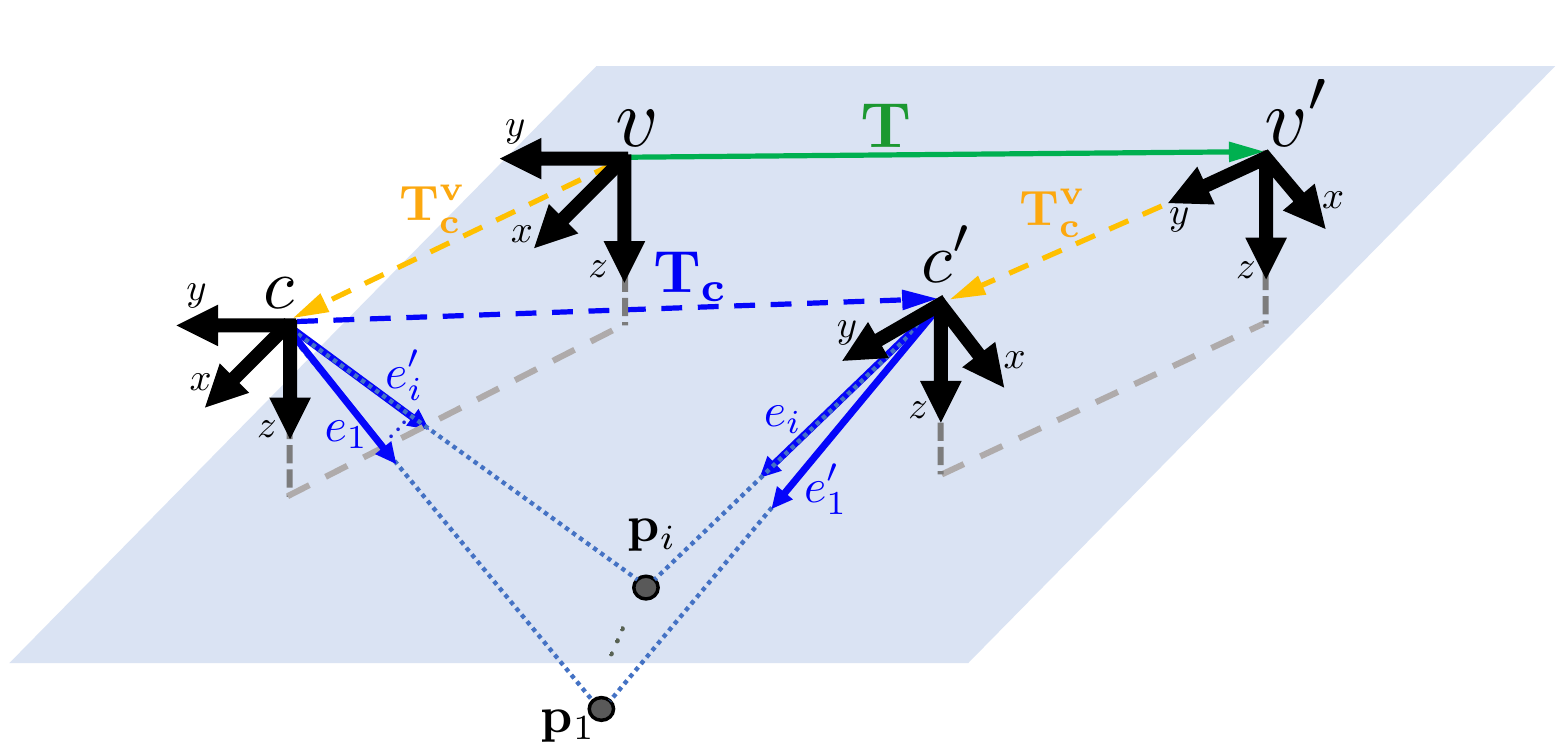}
  \label{Chaining of Transform}
  \caption{Left: \textit{The Ackermann steering model} with the ICR~\cite{ling2020efficient}. Both a left and a right turn are illustrated. Right: Connections between vehicle displacement, extrinsic transformation, and relative camera pose.}
  \label{Ackermann Steering Model + Chaining of Transform}
\end{figure}

\subsection{Homographic Mapping and Bounding Box Extraction}

We rely the globally-optimal BnB solver for correspondence-less AGV motion presented in \cite{ling2020efficient}, which also employs a normal, downward facing camera. We employ the two-dimensional Ackermann steering model describing the commonly non-holonomic motion of an AGV. Employing this 2-DoF model leads to benefits in BnB, the complexity of which strongly depends on the dimensionality of the solution space. As illustrated in Figure~\ref{Ackermann Steering Model + Chaining of Transform}, the Ackermann model constrains the motion of the vehicle to follow a circular-arc trajectory about an Instantaneous Centre of Rotation~(ICR). The motion between successive frames can be conveniently described at the hand of two parameters: the half-angle of the relative rotation angle~$\phi$, and the baseline between the two views~$\rho$. However, the alignment of the events requires a temporal parametrisation of the relative pose, which is why we employ the angular velocity $\omega = \frac{\theta}{t} = \frac{2\phi}{t}$ as well as the translational velocity $v = \omega r = \omega \rho \frac{1}{2\sin(\phi)}$ in our model. The relative transformation from vehicle frame $v'$ back to $v$ is therefore given by
\begin{equation}
  \label{transform_with_respect_to_vehicle}
  \mathbf{R}_v = \left[
                              \begin{matrix}
                                \cos\!\left(\omega t\right) & - \sin\!\left(\omega t\right) & 0 \\
                                \sin\!\left(\omega t\right) &   \cos\!\left(\omega t\right) & 0 \\
                                0 & 0 & 1
                              \end{matrix}
                            \right] \text{ and }
  \mathbf{t}_v = \frac{v}{\omega} \left[
                              \begin{matrix}
                                1 - \cos\!\left(\omega t\right) \\
                                    \sin\!\left(\omega t\right) \\
                                                              0
                              \end{matrix}
                            \right] \,.
\end{equation}
Further details about the derivation are given in the supplementary material.

In practice the vehicle frame hardly coincides with the camera frame. The orientation and the height of the origin can be chosen to be identical, and the camera may be laterally mounted in the centre of the vehicle. However, there is likely to be a displacement along the forward direction, which we denote by the signed variable~$s$. In other words, $\mathbf{R}_{v}^{c}=\mathbf{I}_{3\times 3}$ and $\mathbf{t}_{v}^{c}=\left[ \begin{matrix} 0 & s & 0 \end{matrix} \right]^T$. As illustrated in Figure~\ref{Ackermann Steering Model + Chaining of Transform}, the transformation from camera pose $c'$ (at an arbitrary future timestamp) to $c$ (at the initial timestamp $t_{\mathrm{ref}}$) is therefore given by
\begin{equation}
\label{rotation_matrix_and_translation_vector}
\begin{split}
	& \mathbf{R}_c = \mathbf{R}_{v}^{cT}\mathbf{R}_v \mathbf{R}_{v}^{c} \,, \\
	& \mathbf{t}_c = -\mathbf{R}_{v}^{cT}\mathbf{t}_{v}^{c} + \mathbf{R}_{v}^{cT}\mathbf{t}_v + \mathbf{R}_{v}^{cT}\mathbf{R}_v\mathbf{t}_{v}^{c}\,.
\end{split}
\end{equation}

Using the known plane normal vector~$\mathbf{n}=\left[\begin{matrix}0 & 0 & -1 \end{matrix}\right]^T$ and depth-of-plane~$d$, the image warping function~$W(\mathbf{x}_k,t_k;[\omega\text{ }v]^T)$ that permits the transfer of an event~$e_k = \{ \mathbf{x}_k,t_k,s_k \}$ into the reference view at~$t_{\mathrm{ref}}$ is finally given by the planar homography equation
\begin{equation}
    \label{homography}
    \textbf{H} \left[ \begin{matrix} \mathbf{x}_k^T & 1 \end{matrix}\right]^T = \textbf{K}(\textbf{R}_c-\frac{\textbf{t}_c\textbf{n}^\textbf{T}}{d})\textbf{K}^{\textbf{-1}} \left[ \begin{matrix} \mathbf{x}_k^T & 1 \end{matrix}\right]^T \,.
\end{equation}
Note that $\mathbf{K}$ here denotes a regular perspective camera calibration matrix with homogeneous focal length $f$, zero skew, and a principal point at $\left[ \begin{matrix} u_0 & v_0 \end{matrix} \right]^T$. Note further that the substituted time parameter needs to be equal to $t=t_k-t_{\mathrm{ref}}$, and that the result needs to be dehomogenised. After expansion, we easily obtain
\begin{eqnarray}
  \mathbf{x}^{\prime}_{k} & = & W(\mathbf{x}_k,t_k;[\omega\text{ }v]^T) = \left[ \begin{matrix} x_{k}^{\prime} & y_{k}^{\prime} \end{matrix} \right]^T \\
  & = & \left[ \begin{matrix}
      - [y_k - v_0 + s \frac{f}{d}] \sin(\omega t)
      + [x_k - u_0 - \frac{f}{d} (\frac{v}{w})] \cos(\omega t)
      + \frac{f}{d} (\frac{v}{w}) + u_0 \\
      [x_k - u_0 - \frac{f}{d} (\frac{v}{w})] \sin(\omega t)
      + [y_k - v_0 + s \frac{f}{d}] \cos(\omega t)
      - s \frac{f}{d} + v_0
  \end{matrix} \right] \,. \nonumber
\end{eqnarray}

Finally, the bounding box $\mathcal{P}_{k }^{\boldsymbol{\Theta}}$ is found by bounding the values of $x_{k}^{\prime}$ and $y_{k}^{\prime}$ over the intervals $\omega\in\mathcal{W}=\left[\omega_{\mathrm{min}};\omega_{\mathrm{max}}\right]$ and $v\in\mathcal{V}=\left[v_{\mathrm{min}};v_{\mathrm{max}}\right]$. The bounding is easily achieved if simply considering monotonicity of functions over given sub-branches. For example, if $\omega_{\mathrm{min}} \geq 0$, $v_{\mathrm{min}} \geq 0$, $x_k \geq u_0$, and $y_k \geq v_0 - s \frac{f}{d}$, we obtain
\small
\begin{eqnarray}
    \underline{x_{k}^{\prime}} &=& - [y_k - v_0 + s \frac{f}{d}] \sin(\omega_{\mathrm{max}} t)
                    + [x_k - u_0 - \frac{f}{d} (\frac{v_{\mathrm{min}}}{w_{\mathrm{max}}})] \cos(\omega_{\mathrm{max}} t)
                    + \frac{f}{d} (\frac{v_{\mathrm{min}}}{w_{\mathrm{max}}}) + u_0 \,, \nonumber \\
    \overline{x_{k}^{\prime}}  &=& - [y_k - v_0 + s \frac{f}{d}] \sin(\omega_{\mathrm{min}} t)
                    + [x_k - u_0 - \frac{f}{d} (\frac{v_{\mathrm{max}}}{w_{\mathrm{min}}})] \cos(\omega_{\mathrm{min}} t)
                    + \frac{f}{d} (\frac{v_{\mathrm{max}}}{w_{\mathrm{min}}}) + u_0 \,, \nonumber \\
    \underline{y_{k}^{\prime}} &=&   [x_k - u_0 - \frac{f}{d} (\frac{v_{\mathrm{max}}}{w_{\mathrm{min}}})] \sin(\omega_{\mathrm{min}} t)
                    + [y_k - v_0 + s \frac{f}{d}] \cos(\omega_{\mathrm{max}} t)
                    - s \frac{f}{d} + v_0 \,, \text{ and} \nonumber \\
    \overline{y_{k}^{\prime}} &=&    [x_k - u_0 - \frac{f}{d} (\frac{v_{\mathrm{min}}}{w_{\mathrm{max}}})] \sin(\omega_{\mathrm{max}} t)
                    + [y_k - v_0 + s \frac{f}{d}] \cos(\omega_{\mathrm{min}} t)
                    - s \frac{f}{d} + v_0 \,.
\end{eqnarray}
\normalsize
We kindly refer the reader to the supplementary material for all further cases.

\section{Experimental evaluation}

We present two suites of experiments. The first one validates the global optimality, accuracy and robustness of our solver on simulated data. The second one then applies it to the real-world scenario of AGV motion estimation.

\subsection{Accuracy and Robustness of Globally Optimal Motion Estimation}

We start by evaluating the accuracy of the motion estimation with contrast maximisation function $L_{\mathrm{SoS}}$ over synthetic data. As already implied in~\cite{gallego2019focus}, $L_{\mathrm{SoS}}$ can be considered as a solid starting point for the evaluation. Our synthetic data consists of randomly generated horizontal and vertical line segments on a plane at a depth of 2.0m. We consider Ackermann motion with an angular velocity $\omega = 28.6479^{\circ}$/s (0.5rad/s) and a linear velocity $v = 0.5$m/s. Events are generated by randomly choosing a 3D point on a line, and reprojecting it into a random camera pose sampled by a random timestamp within the interval $[0, 0.1s]$. The result of our method is finally evaluated by running BnB over the search space $\mathcal{W}=[0.4,0.6]$ and $\mathcal{V}=[0.4,0.6]$, and comparing the retrieved solution against the result of an exhaustive search with sampling points every $\delta \omega=0.001$rad/s and $\delta v=0.001$m/s. BnB is furthermore configured to terminate the search if $|\omega_{max}-\omega_{min}| \leq 0.00078$rad/s or $|v_{max}-v_{min}| \leq 0.00078$m/s. The experiment is repeated 1000 times.

Figures~\ref{fig:CM_Error_w} and \ref{fig:CM_Error_v} illustrate the distribution of the errors for both methods in the noise-free case. The standard deviation of the exhaustive search and BnB are $\sigma_{\omega}=1.0645^{\circ}$/s, $\sigma_{v}=0.0151$m/s and $\sigma_{\omega}=1.305^{\circ}$/s, $\sigma_{v}=0.0150$m/s, respectively. While this result suggests that BnB works well and sustainably returns a result very close to the optimum found by exhaustive search, we still note that the optimum identified by both methods has a bias with respect to ground truth, even in the noise-free case. Note however that this is related to the nature of the contrast maximisation function, and not our globally optimal solution strategy.

\begin{figure}[t]
\centering
\subfigure[]{
\includegraphics[width=0.22\textwidth]{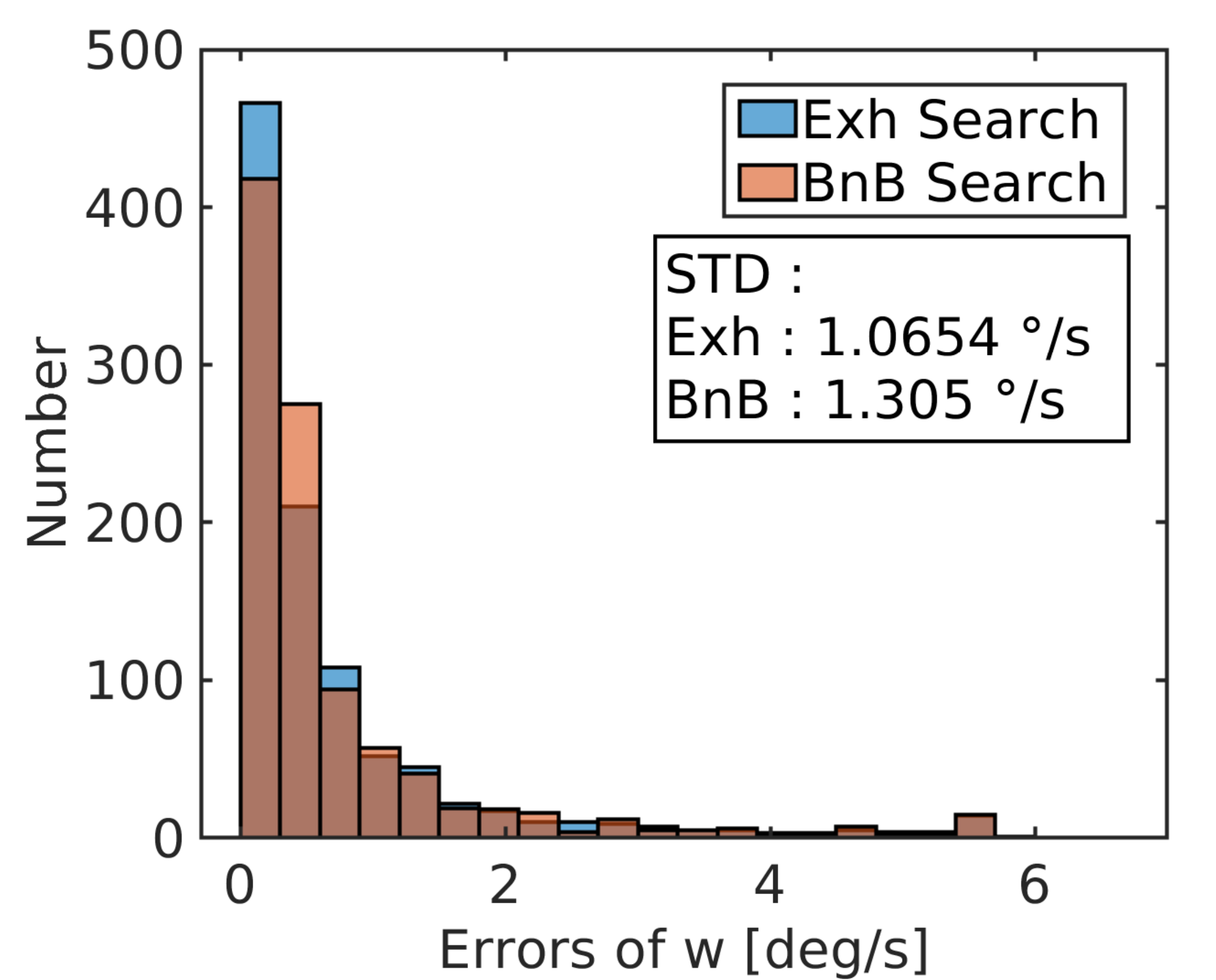} 
\label{fig:CM_Error_w}
}
\subfigure[]{
\includegraphics[width=0.22\textwidth]{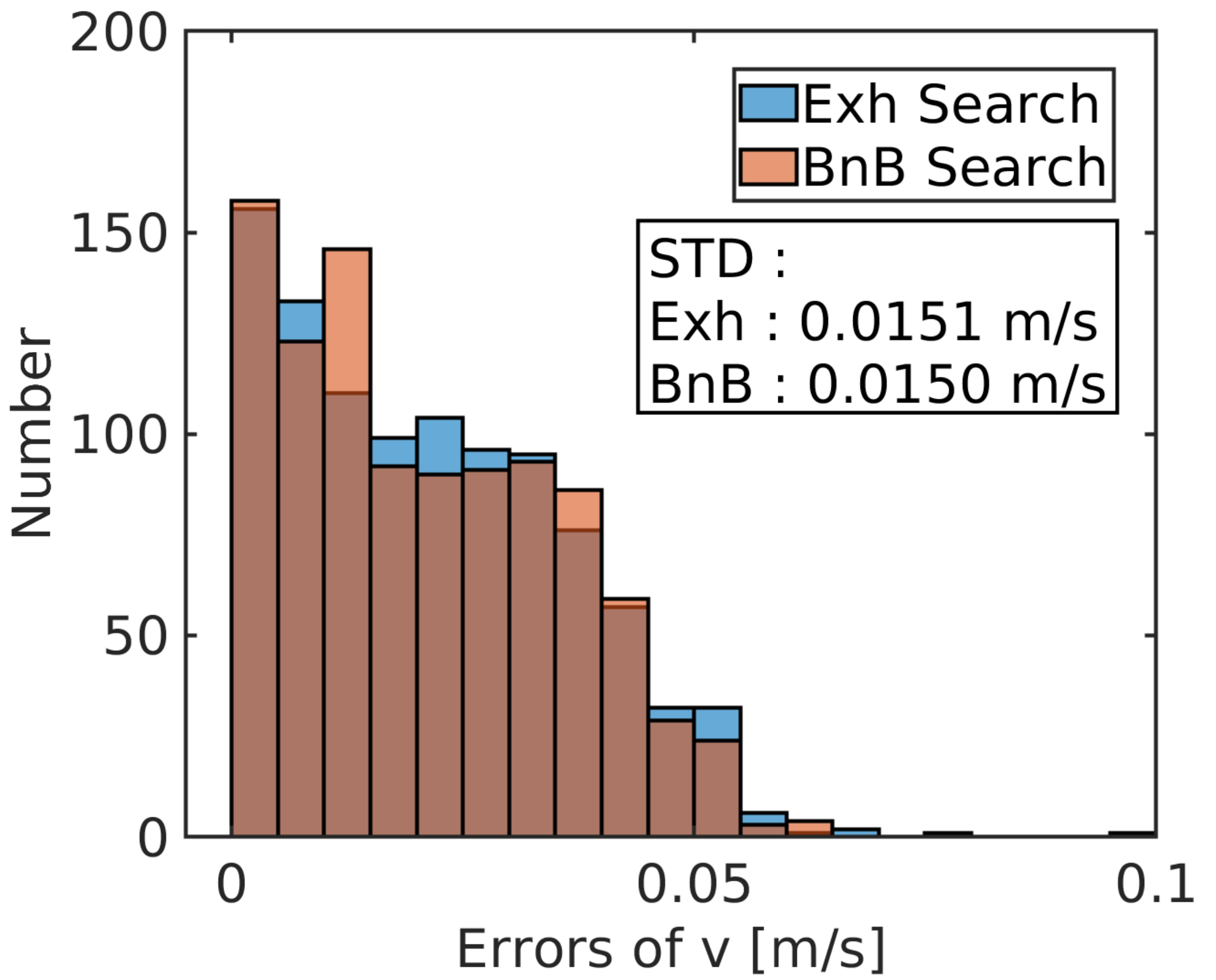} 
\label{fig:CM_Error_v}
}
\subfigure[]{
\includegraphics[width=0.22\textwidth]{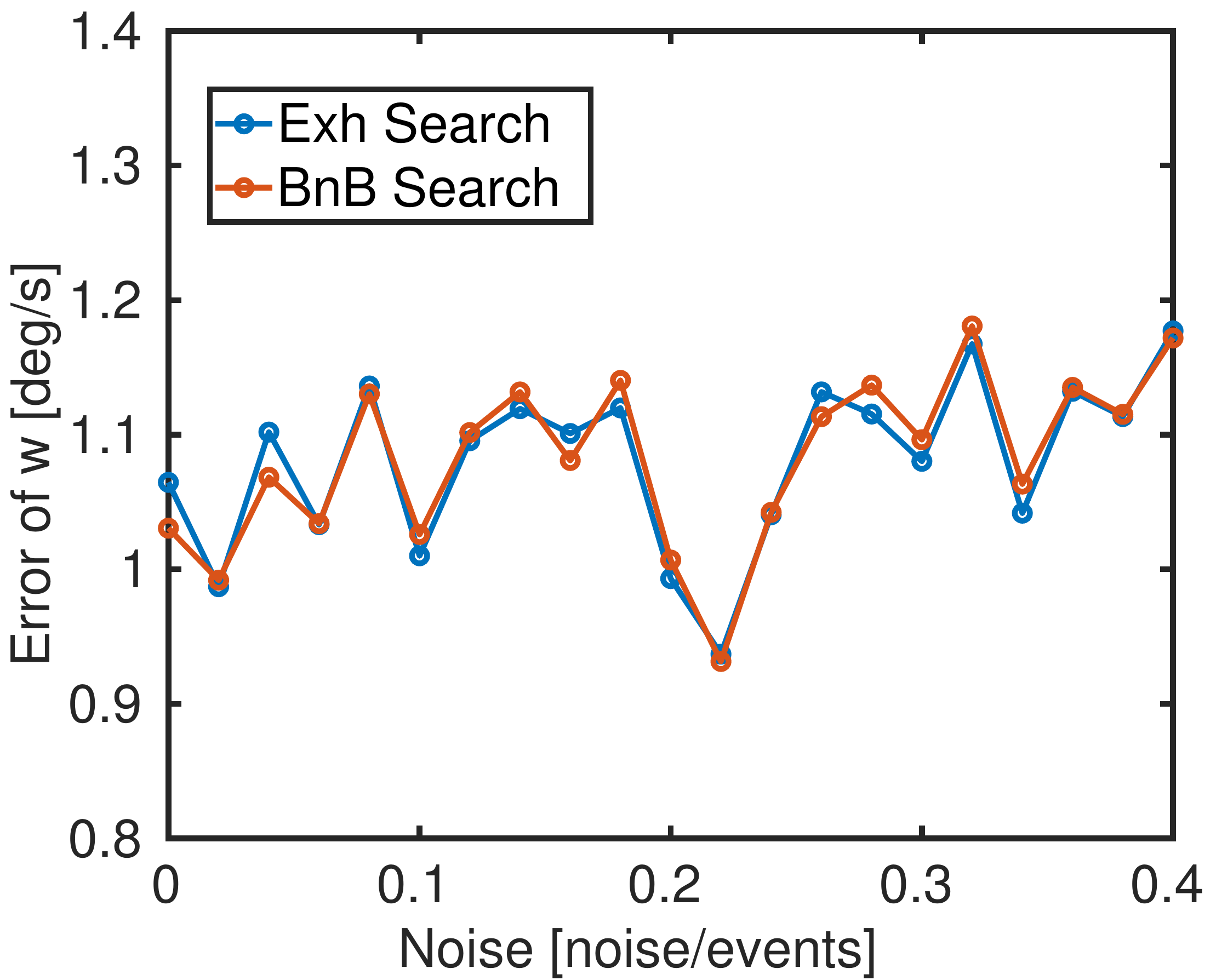} 
\label{fig:CM_Error_Noise_w}
}
\subfigure[]{
\includegraphics[width=0.22\textwidth]{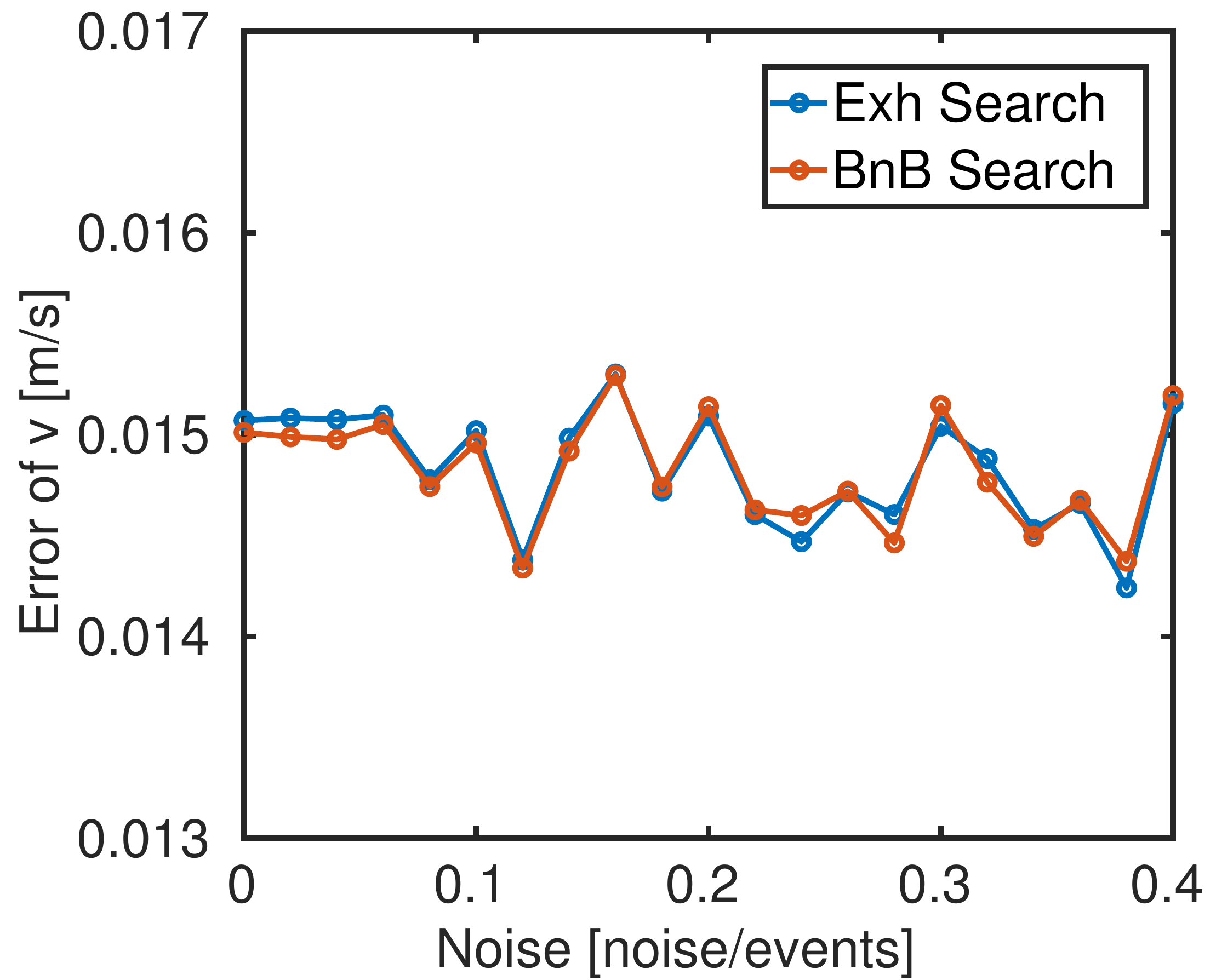}
\label{fig:CM_Error_Noise_v}
}
\caption{Simulation Results. (a) and (b) indicate the error distribution for $\omega$ and $v$ over all experiments for both our proposed method as well as an exhaustive search. (c) and (d) visualise the average error of the estimated parameters caused by additional salt and pepper noise on the event stream. Results are averaged over 1000 random experiments. Note that our proposed method has excellent robustness even for N/E ratios up to 40\%.}
\label{fig:simulation results}
\end{figure}

In order to analyse robustness, we randomly add salt and pepper noise to the event stream with noise-to-event (N/E) ratios between 0 and 0.4 (Example objective functions for different N/E ratios have already been illustrated in Figure~\ref{fig:heat map}). Figure~\ref{fig:CM_Error_Noise_w} and \ref{fig:CM_Error_Noise_v} show the error for each noise level again averaged over 1000 experiments. As can be observed, the errors are very similar and behave more or less independently of the amount of added noise. The latter result underlines the high robustness of our approach.

\subsection{Application to real data and comparison against alternatives}

We apply our method to real data collected by a DAVIS346 event camera, which outputs events streams with a maximum time resolution of 1$\mu$s as well as regular frames at a frame rate of 30Hz. Images have a resolution of 346$\times$260. We mount the camera on the front of a XQ-4 Pro robot and let it face downward. The displacement from the non-steering axis to the camera is $s = -0.45$m, and the height difference between camera and ground is $d = 0.23$m. We recorded several motion sequences on a wood grain foam which has highly self-similar texture and poses a challenge to reliably extract and match features. Ground truth is obtained via an Optitrack optical motion tracking system. Our algorithm is working in undistorted coordinates, which is why normalisation and undistortion are computed in advance. The following aspects are evaluated:

\textbf{Different objective functions}:
We test the algorithm with all aforementioned six contrast functions over various types of motions, including a straight line, a circle, and an arbitrarily curved trajectory. Table~\ref{tab:errors} shows the RMS errors of the estimated dynamic parameters, and compares the accuracy of all six alternatives. We furthermore apply two state-of-the-art approaches for regular images, namely the correspondence-less globally optimal feature-based approach (GOVO) from~\cite{ling2020efficient}, as well as the Improved Fourier Mellin Invariant transform~(IFMI) in~\cite{xu2019improved,bulow2009fast}. Even though these alternatives use the same non-holonomic or planar motion models, event-based motion estimation methods significantly outperform the intensity-camera-based alternatives ($L_{\mathrm{SoSAaS}}$ on top, and $L_{\mathrm{SoS}}$ and $L_{\mathrm{Var}}$ also have good performance).
\begin{table}[t]
\footnotesize
\caption{RMS errors for different datasets and methods.}
\centering
\begin{tabular}{|c|c|c|c|c|c|c|}
\hline
\textbf{\ Method\ } 
& \textbf{\begin{tabular}[c]{@{}c@{}}Line\\ w{[}$^{\circ}$/s{]}\end{tabular}} 
& \textbf{\begin{tabular}[c]{@{}c@{}}Line\\ v{[}m/s{]}\end{tabular}} 
& \textbf{\begin{tabular}[c]{@{}c@{}}Circle\\ w{[}$^{\circ}$/s{]}\end{tabular}} 
& \textbf{\begin{tabular}[c]{@{}c@{}}Circle\\ v{[}m/s{]}\end{tabular}} 
& \textbf{\begin{tabular}[c]{@{}c@{}}Curve\\ w{[}$^{\circ}$/s{]}\end{tabular}} 
& \textbf{\begin{tabular}[c]{@{}c@{}}Curve\\ v{[}m/s{]}\end{tabular}} \\ 
\hline
SoE     & 2.4089      & 0.0158     & 2.2121     & 0.0252     & 3.6282 & 0.0263  \\
SoEaS   & 2.4057      & 0.0158    & 2.0178     & 0.0242     & 3.6282 & 0.0263  \\
SoS     & \textbf{0.5127}      &\ \textbf{0.0086}\ \ & 1.0884  & 0.0083    & 3.0091     & 0.0208  \\
SoSA    & 1.9606      & 0.0287           & 4.2496   & 0.0734    & 9.2904    & 0.0727  \\
SoSAaS  & \textbf{0.5175}  & \textbf{0.0086}   &\ \textbf{0.5294}\ \ &\ \textbf{0.0046}\ \ &\ \textbf{0.5546}\ \ &\ \textbf{0.0189}\ \ \\
Var     & \textbf{0.5127}      & \textbf{0.0086}   & 1.0884   & 0.0083    & 3.0091   & 0.0208  \\
IFMI    & 145.3741    & 1.0594     & 8.1092    & 0.0243    & 12.8047    & 0.0192   \\
GOVO   & 6.9705    & 0.2409     & 4.5506   & 0.0642   & 9.8652  & 0.0590         \\ \hline
\end{tabular}
\label{tab:errors}
\end{table}

\textbf{Event-based vs frame-based}:
GOVO~\cite{ling2020efficient} and IFMI~\cite{xu2019improved} are frame-based algorithms specifically designed for planar AGV motion estimation under featureless conditions. Figure~\ref{fig:robot and frame} shows an example frame of the wood grain foam texture, and Figure~\ref{fig:Real_Data} the results obtained for all methods. As can be observed, GOVO finds as little as three corner features for some of the images, thus making it difficult to accurately recover the vehicle displacement despite the globally-optimal correspondence-less nature of the algorithm. Both IFMI and GOVO occasionally lose tracking (especially for linear motion), which leaves our proposed globally-optimal event-based method using $L_{\mathrm{SoSAaS}}$ as the best method.

\begin{figure}[b]
\subfigure
{
\begin{minipage}{0.55\textwidth}
\ \ \includegraphics[width=\textwidth]{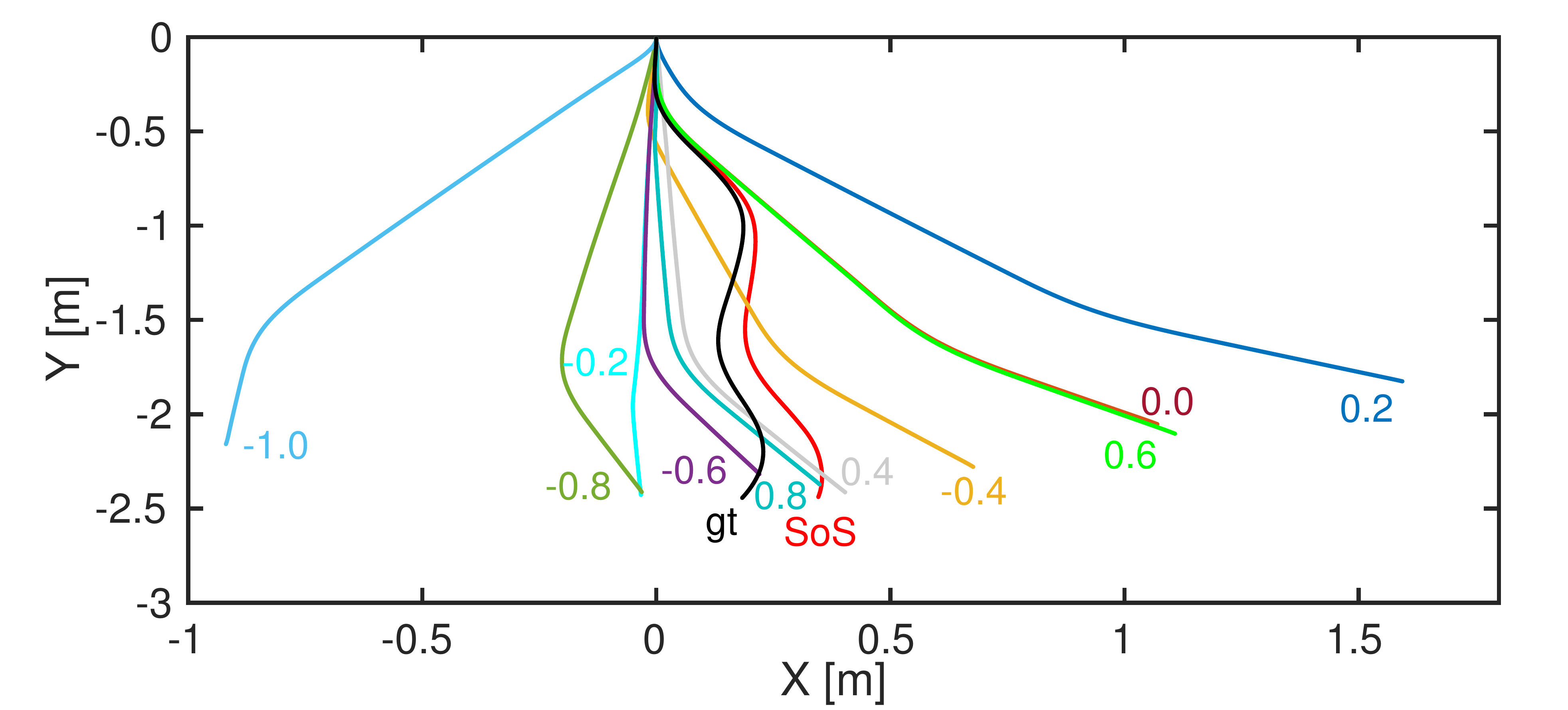}
\end{minipage}
}
\subfigure
{
\begin{minipage}{0.3\textwidth}
\renewcommand\arraystretch{1.2}
\ \ \ \ \ \ \ \ \begin{tabular}{|c|c|c|}
\hline
\bf\ Method\ \ & \bf \begin{tabular}[c]{@{}c@{}}\ w{[}$^{\circ}$/s{]}\end{tabular}\ \ & \bf \begin{tabular}[c]{@{}c@{}}\ v{[}m/s{]}\end{tabular}\ \ \\ \hline
\bf SoS & 3.0091 & 0.0208 \\ \hline
\bf GA & 11.5023 & 0.0379 \\
\hline
\end{tabular}
\end{minipage}
}
\caption{Estimated trajectories by our method (SoS), gradient ascent with various initializations, and ground truth (gt). The table indicates the RMS errors for the best performing gradient ascent run and SoS.}
\label{fig:gd_tra}
\end{figure}
\textbf{BnB vs Gradient Ascent}: We apply both gradient descent as well as BnB to the \textit{Foam} dataset with curved motion. For the first temporal interval and the local search method, we vary the initial angular velocity $\omega$ and linear velocity $v$ between -1 and 0.8 with steps of 0.2 (rad/s or m/s, respectively). For later intervals, we use the previous local optimum. Figure~\ref{fig:gd_tra} illustrates the estimated trajectories for all initial values, compared against ground truth and a BnB search using $L_{\mathrm{SoS}}$. RMS errors are also indicated. As clearly shown, even the best initial guess eventually diverges under a local search strategy, thus leading to clearly inferior results compared to our globally optimal search.
\begin{figure}[t!]
\centering
\includegraphics[width=0.29\textwidth]{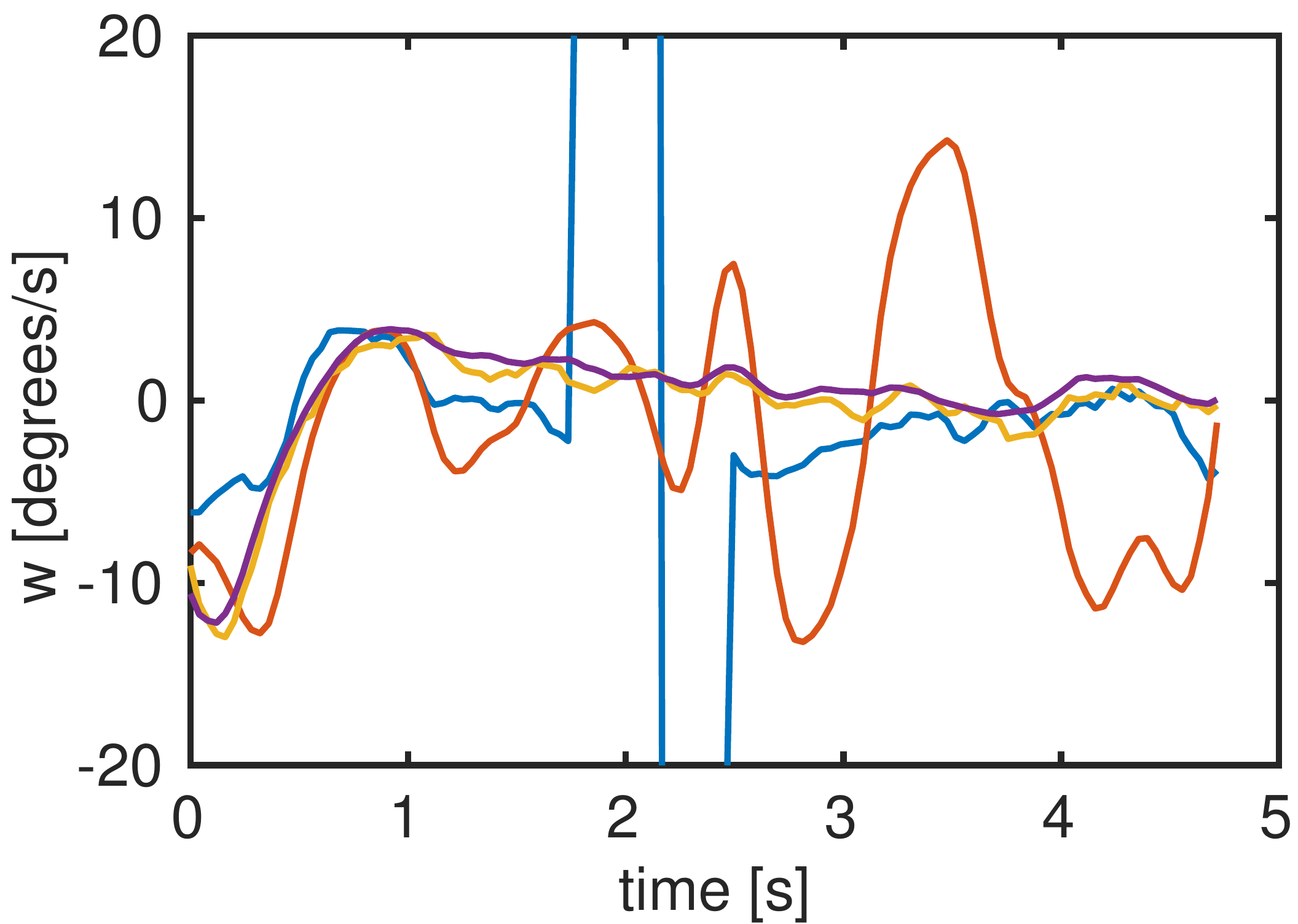}
\includegraphics[width=0.29\textwidth]{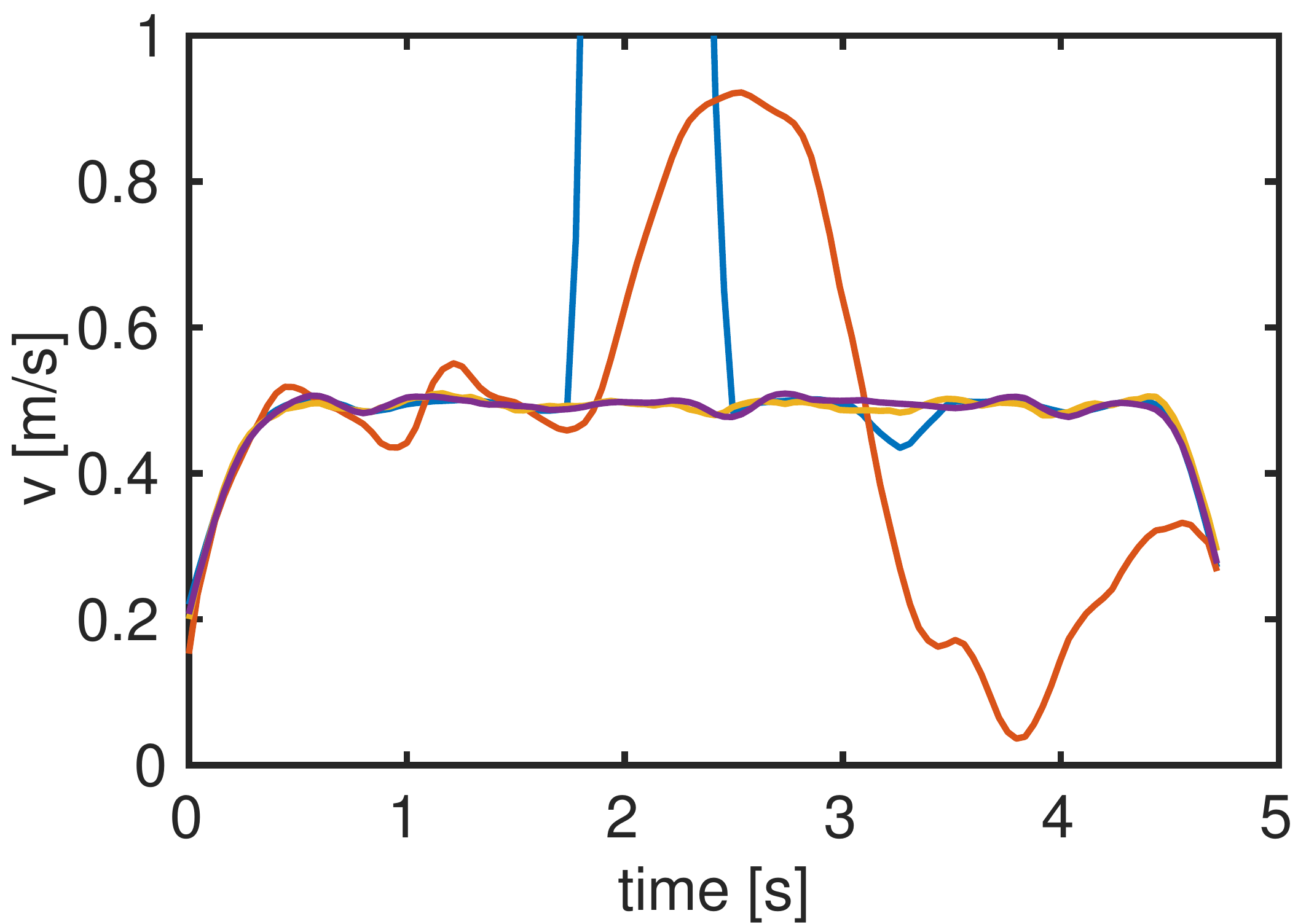}
\includegraphics[width=0.29\textwidth]{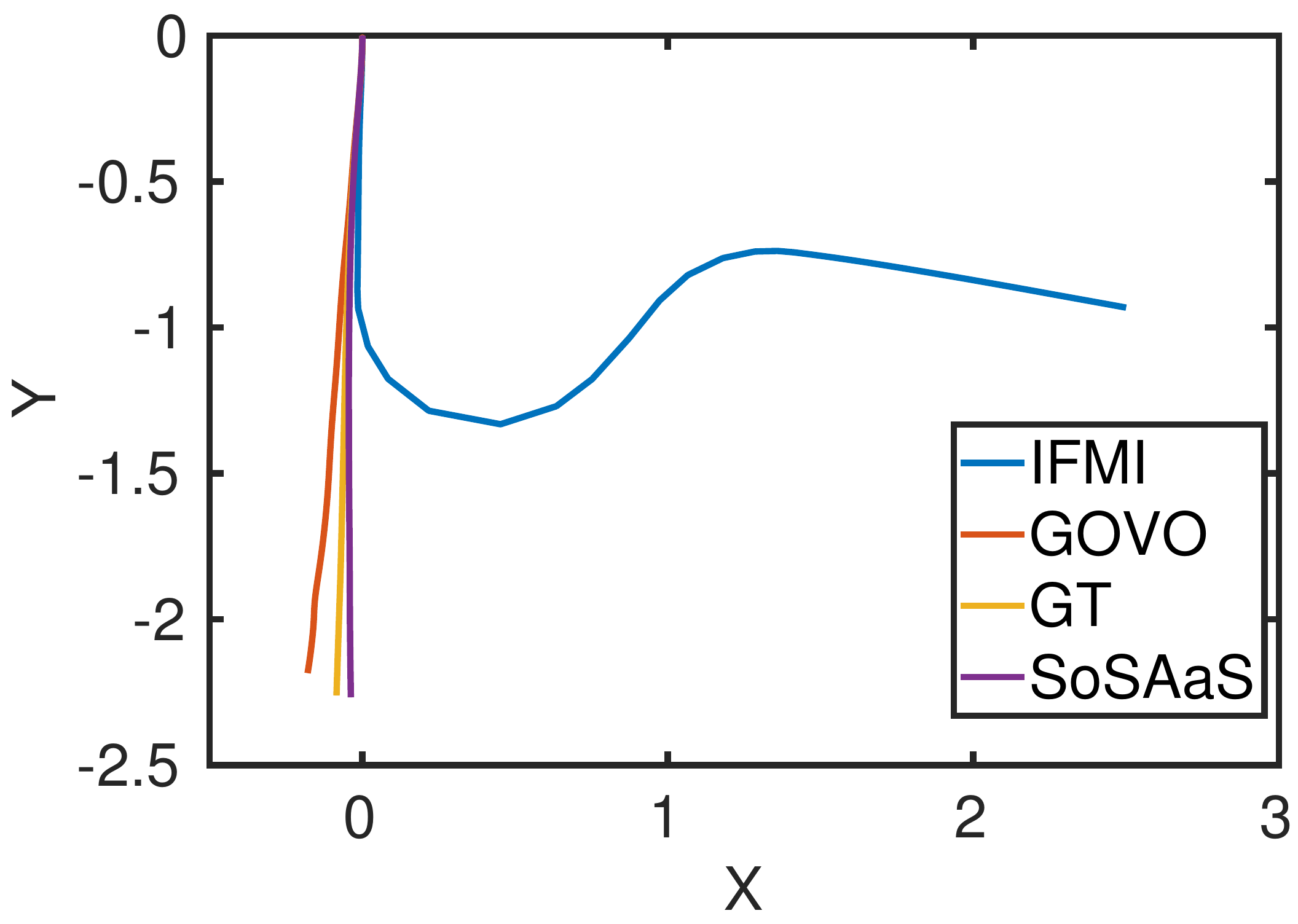} \\
\includegraphics[width=0.29\textwidth]{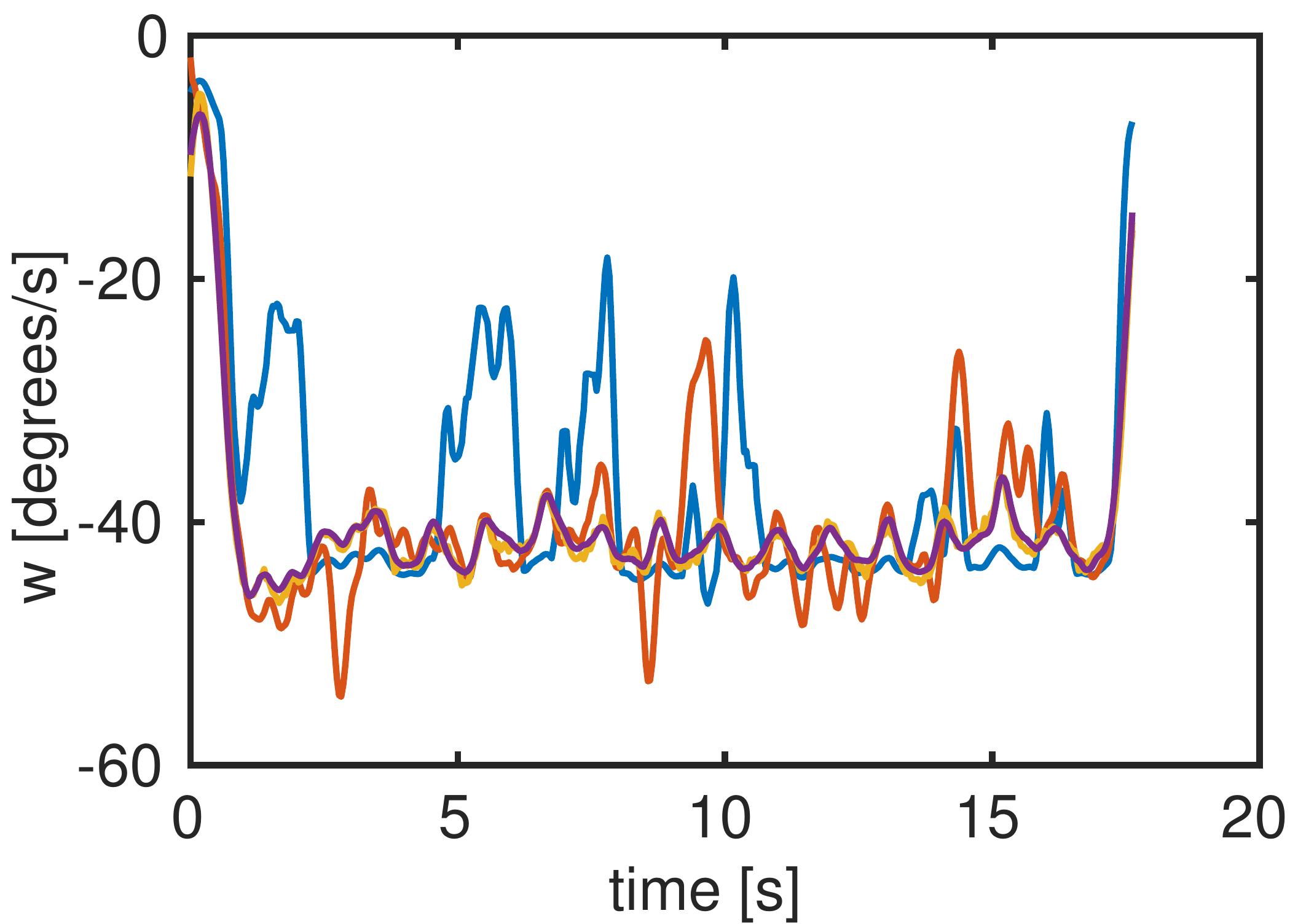}
\includegraphics[width=0.29\textwidth]{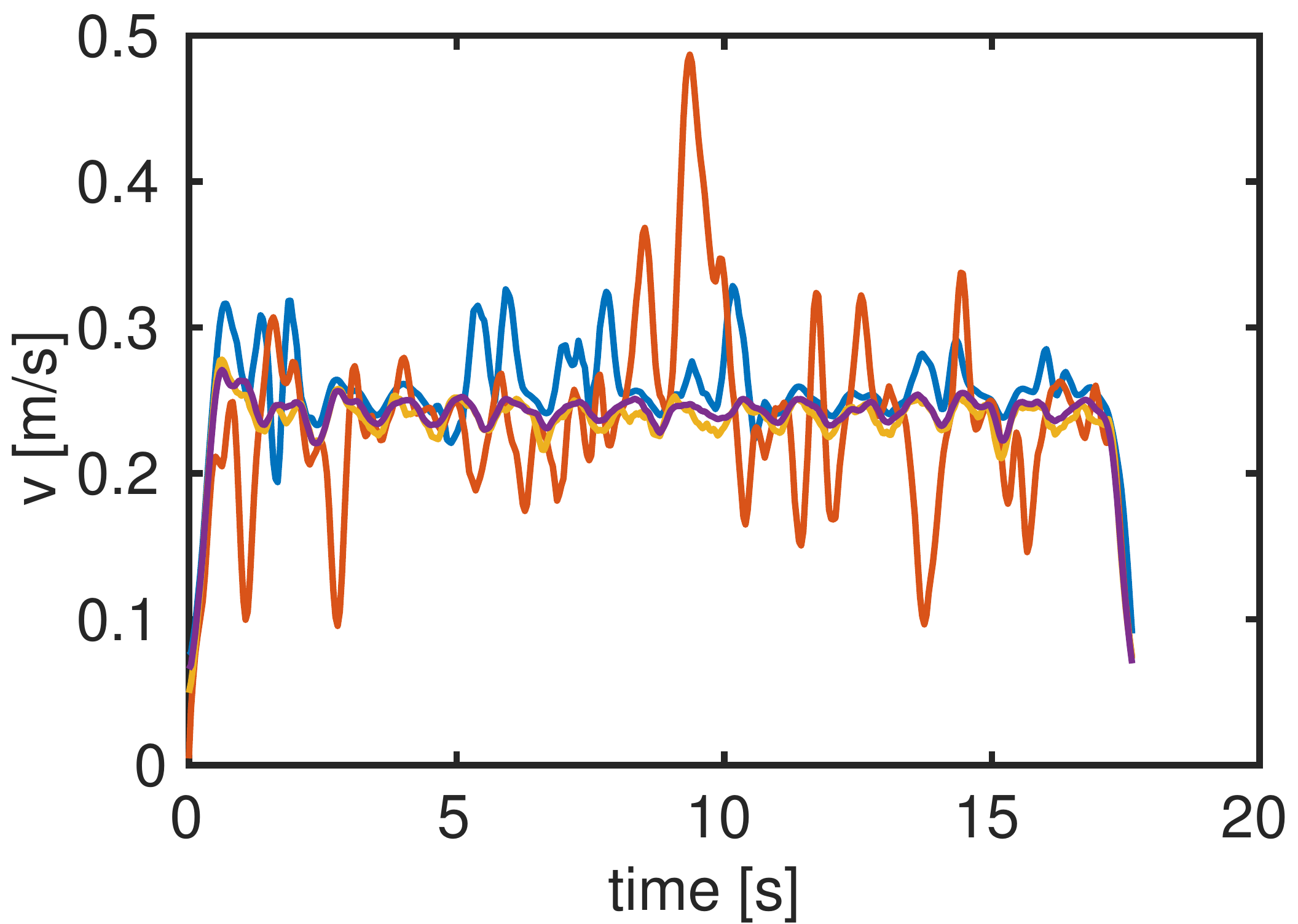}
\includegraphics[width=0.29\textwidth]{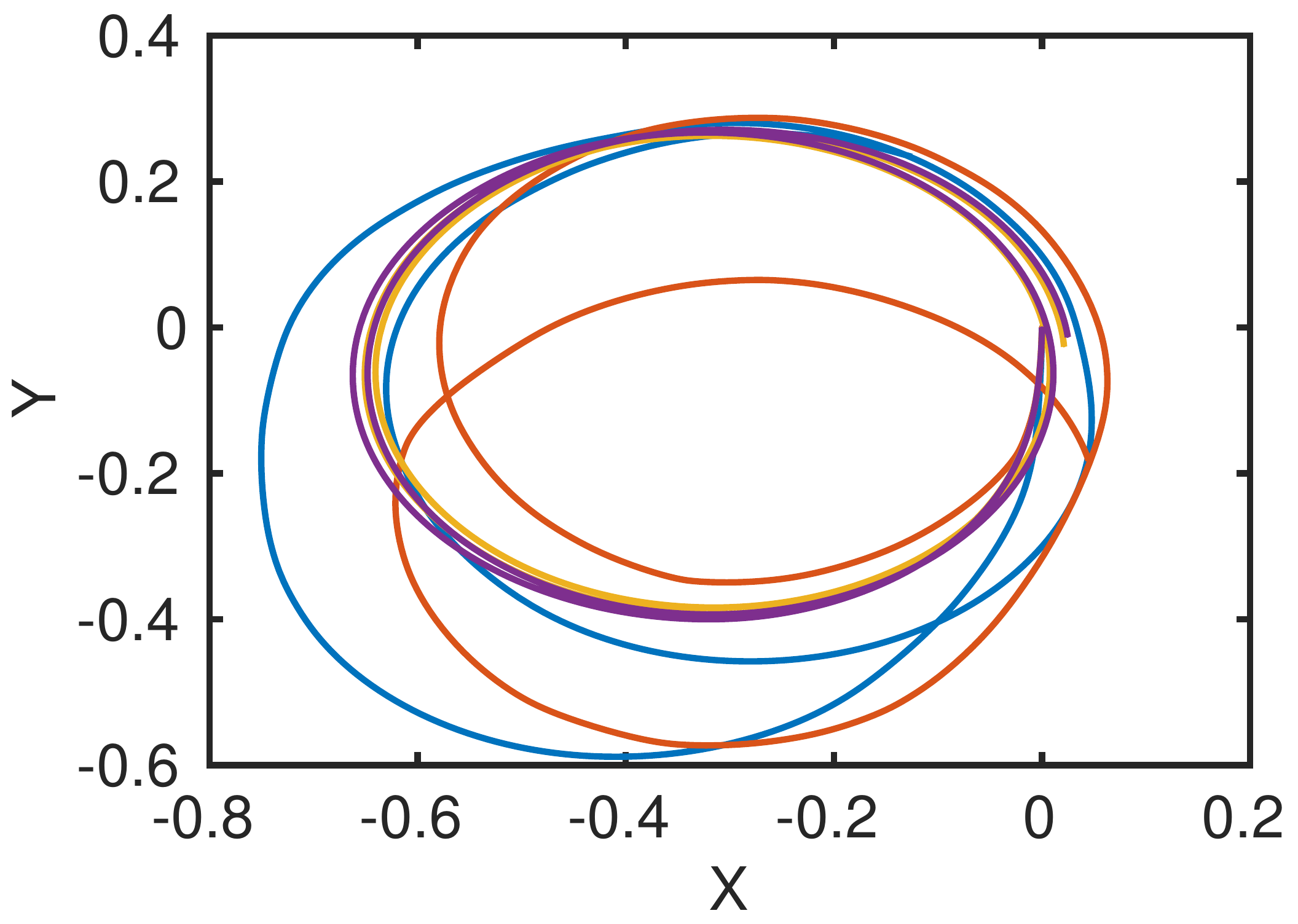} \\
\includegraphics[width=0.29\textwidth]{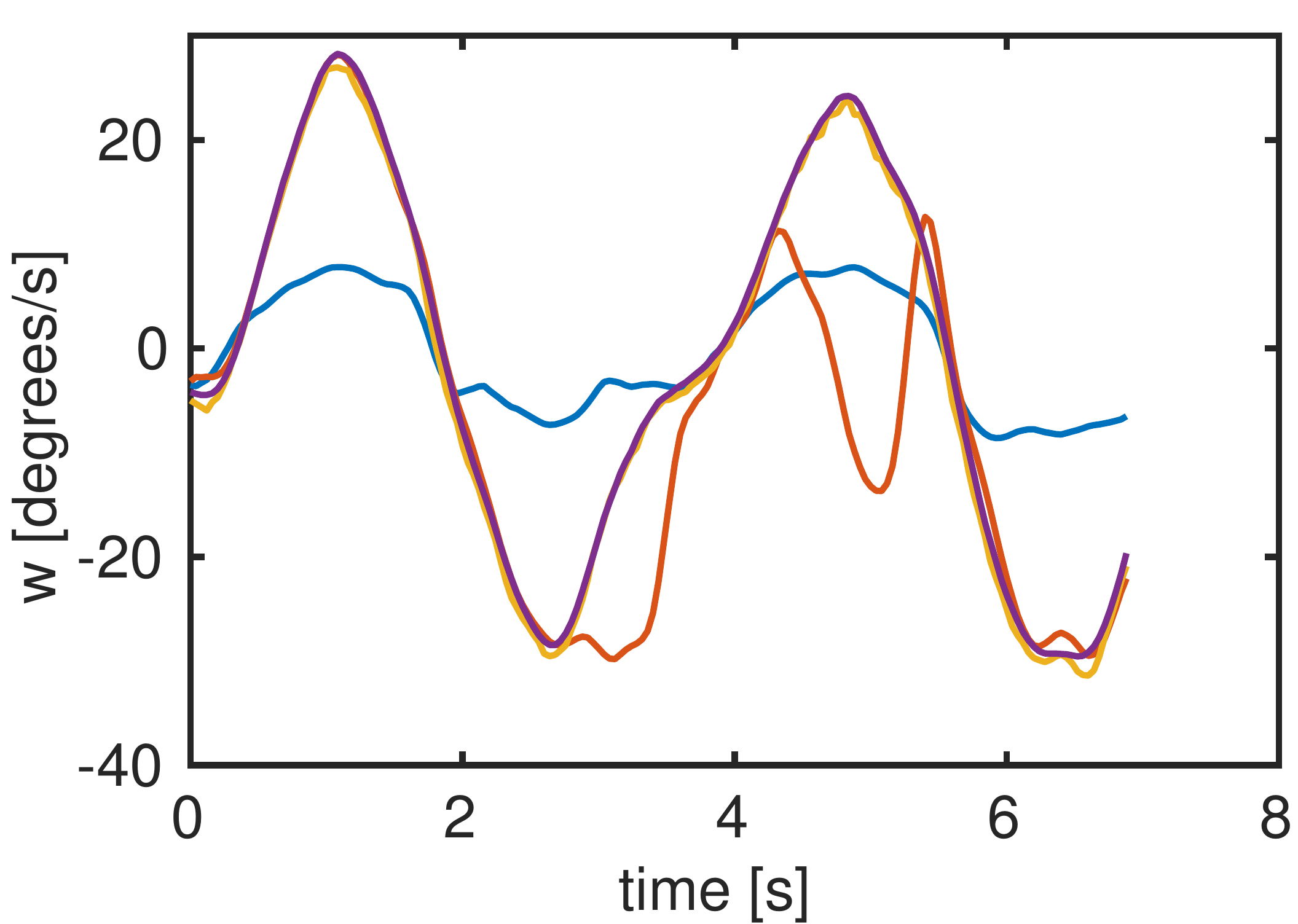}
\includegraphics[width=0.29\textwidth]{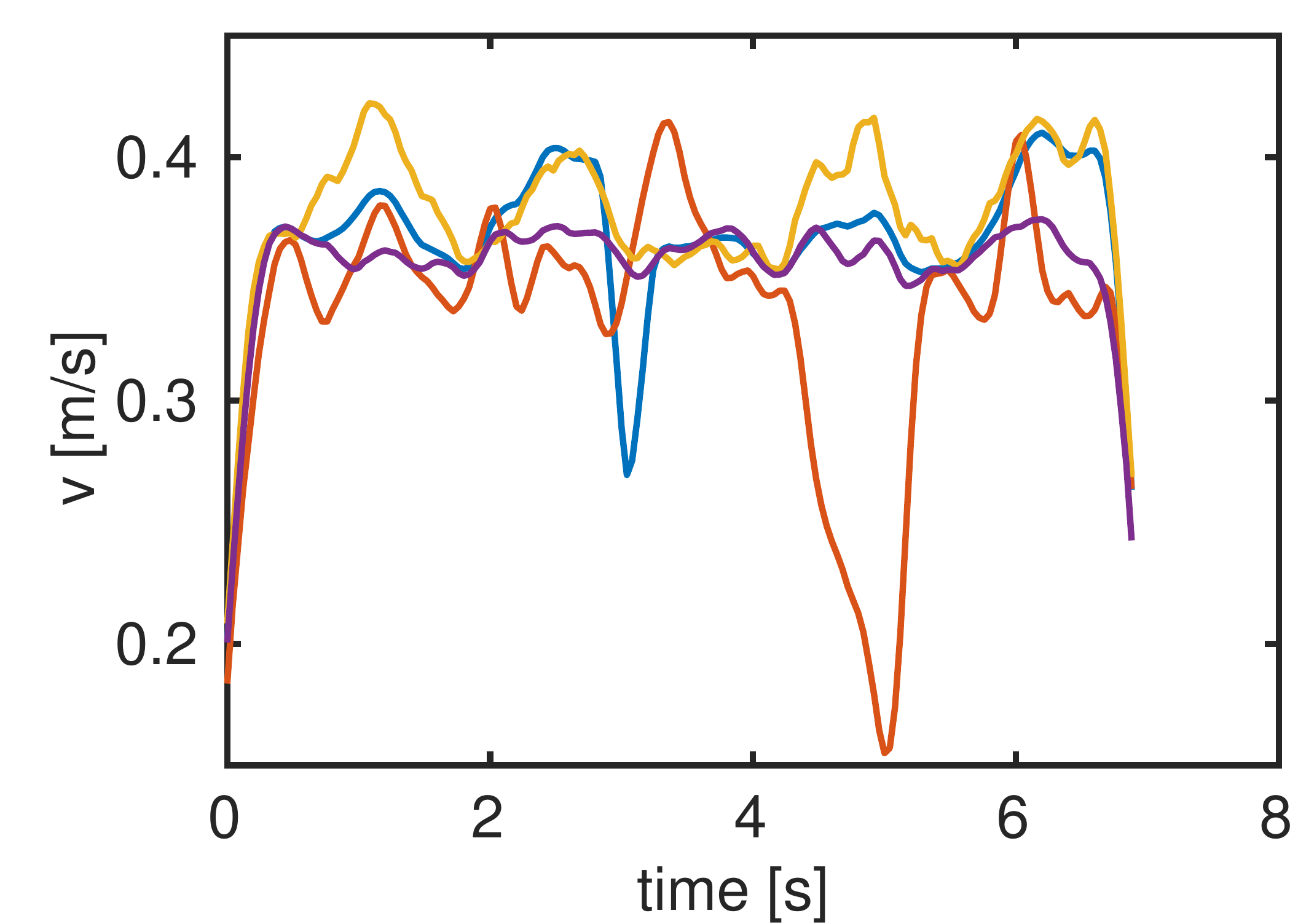}
\includegraphics[width=0.29\textwidth]{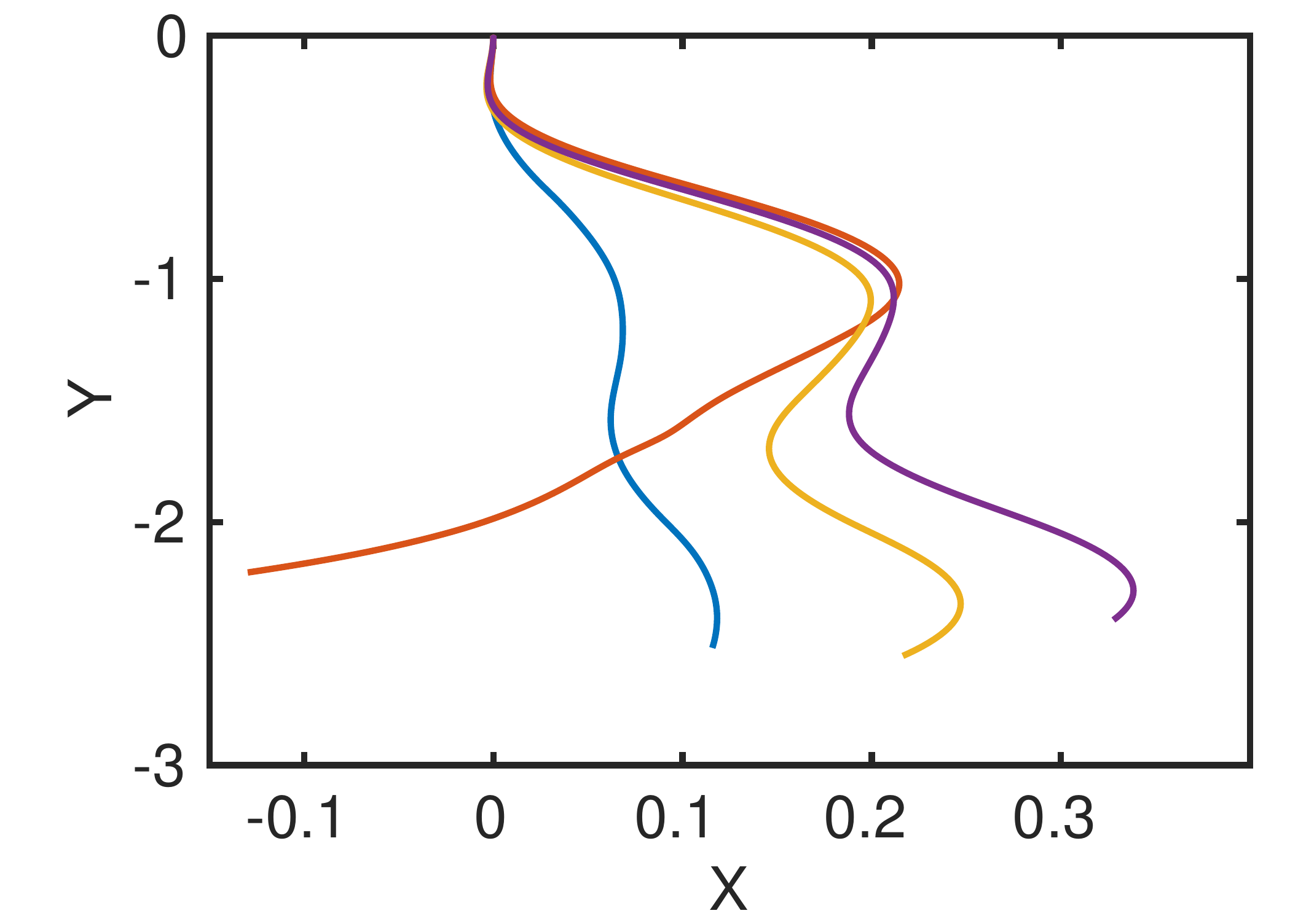} 
\caption{Results for all methods over different datasets. The first two columns are errors over time for $\omega$ and $v$, and the third column illustrates a bird's eye view onto the integrated trajectories.}
\label{fig:Real_Data}
\end{figure}

\textbf{Various textures}: More results over datasets with other ground floor textures can be found in the supplementary material.


\section{Discussion}

We have introduced the first globally optimal solution to contrast maximisation for un-warped event streams. To the best of our knowledge, we are also the first to apply the idea of homography estimation via contrast maximisation to the real-world case of non-holonomic motion estimation with a downward facing camera mounted on an AGV. The challenging conditions in this scenario favorise dynamic vision sensors over regular frame-based cameras, a claim that is supported by our experimental results. The latter furthermore prove that global solutions are important and significantly outperform incremental local refinement. The recursive formulation of our bounds lets us find the global optimum over event streams of 0.04s within less than one minute, a respectable achievement given the typically low computational efficiency of BnB solvers.

\section*{Acknowledgments}
The authors would like to thank the fundings sponsored by Natural Science Foundation of Shanghai (grant number: 19ZR1434000) and Natural Science Foundation of China (grant number: 61950410612).


\end{document}


\pagestyle{headings}
\mainmatter
\def\ECCVSubNumber{5299}  

\title{Globally-Optimal Event Camera\\Motion Estimation\\--Supplementary Material--}


\titlerunning{Globally-Optimal Event Camera Motion Estimation}

\author{}

\authorrunning{X. Peng, Y. Wang, L. Gao, L. Kneip.}

\institute{}


\maketitle


\renewcommand\thesection{\Alph{section}}
\section{Proof of Recursive Upper and Lower Bounds}
\subsection{Variance (Var)}
\begin{equation}
    \begin{split}
        l_{N}(\boldsymbol{\theta}) 
        & =  \frac{1}{N_{p}}\sum_{\mathbf{p}_{ij}\in\mathcal{P}}(I(\mathbf{p}_{ij};\boldsymbol{\theta}) - \mu_{I})^2 \\
        & = \frac{1}{N_{p}} \sum_{\mathbf{p}_{ij}\in\mathcal{P}} \left[ \sum_{k=1}^{N} \boldsymbol{1}(\textbf{p}_{ij} -  W(\mathrm{\bold{x}}_k,t_{k};\boldsymbol{\theta}) ) - \mu_I \right]^2  \\
        & = \frac{1}{N_{p}} \sum_{\mathbf{p}_{ij}\in\mathcal{P}} \left[ \sum_{k=1}^{N-1} \boldsymbol{1}(\textbf{p}_{ij} -  W(\mathrm{\bold{x}}_k,t_{k};\boldsymbol{\theta}) ) - \mu_I + \boldsymbol{1}(\textbf{p}_{ij} - W(\mathrm{\bold{x}}_N,t_{N};\boldsymbol{\theta}) ) \right]^2 \\
        & = l_{N-1}(\boldsymbol{\theta}) + a + b + c,
    \end{split}
\end{equation}
where $\mu_{I} = N/N_p$ is the mean value of $I(\mathbf{p}_{ij};\boldsymbol{\theta})$ over all pixels (a function of $\boldsymbol{\theta}$ itself), which is constant. $N_{p}$ the total number of accumulators in $I$. And
\begin{equation}
    \begin{split}
        & a = \frac{2}{N_{p}} \sum_{\mathbf{p}_{ij}\in\mathcal{P}} \left\{ \boldsymbol{1}(\textbf{p}_{ij} - W(\mathrm{\bold{x}}_N,t_{N};\boldsymbol{\theta}) ) \left[ \sum_{k=1}^{N-1} \boldsymbol{1}(\textbf{p}_{ij} -  W(\mathrm{\bold{x}}_k,t_{k};\boldsymbol{\theta}) ) \right]  \right\} \\
        & b = - \frac{2  \mu_I}{N_{p}} \sum_{\mathbf{p}_{ij}\in\mathcal{P}}  \boldsymbol{1}(\textbf{p}_{ij} - W(\mathrm{\bold{x}}_N,t_{N};\boldsymbol{\theta}) ) = - \frac{2  \mu_I}{N_{p}}  \\
        & c = \frac{1}{N_{p}} \sum_{\mathbf{p}_{ij}\in\mathcal{P}} \left[ \boldsymbol{1}(\textbf{p}_{ij} - W(\mathrm{\bold{x}}_N,t_{N};\boldsymbol{\theta}) ) \right]^2 = \frac{1}{N_p}.
    \end{split}
\end{equation}
Thus, similar to the proof in the paper, for the objective function $L_{N} = \max_{\boldsymbol{\theta} \in \boldsymbol{\Theta}} l_N$, we have
\begin{equation}
    \begin{split}
        & \overline{L_{N}} = \overline{L_{N-1}} + \frac{1}{N_{p}} - \frac{2 \mu_{I}}{N_{p}} + \frac{2}{N_{p}} Q ,  \\
        & \underline{L_{N}} = \underline{L_{N-1}} + \frac{1}{N_{p}} - \frac{2 \mu_{I}}{N_{p}} + \frac{2}{N_{p}}I^{N-1} (\boldsymbol{\eta}_N^{\theta_0};\boldsymbol{\theta}_0) .
    \end{split}
\end{equation}

\subsection{Sum of Exponentials (SoE)}

\begin{equation}
    \begin{split}
        l_{N}(\boldsymbol{\theta}) 
        & = \sum_{\mathbf{p}_{ij}\in\mathcal{P}}e^{I(\mathbf{p}_{ij};\boldsymbol{\theta})} \\
        & = \sum_{\mathbf{p}_{ij}\in\mathcal{P}}e^{\sum_{k=1}^{N} \boldsymbol{1}(\textbf{p}_{ij} -  W(\mathrm{\bold{x}}_k,t_{k};\boldsymbol{\theta}) )} \\
        & = \sum_{\mathbf{p}_{ij}\in\mathcal{P}}e^{\sum_{k=1}^{N-1} \boldsymbol{1}(\textbf{p}_{ij} -  W(\mathrm{\bold{x}}_k,t_{k};\boldsymbol{\theta}) ) + \boldsymbol{1}(\textbf{p}_{ij} - W(\mathrm{\bold{x}}_N,t_{N};\boldsymbol{\theta}) ) } \\
        & = \sum_{\mathbf{p}_{ij}\in\mathcal{P}}e^{\sum_{k=1}^{N-1} \boldsymbol{1}(\textbf{p}_{ij} -  W(\mathrm{\bold{x}}_k,t_{k};\boldsymbol{\theta}) )} e^{ \boldsymbol{1}(\textbf{p}_{ij} - W(\mathrm{\bold{x}}_N,t_{N};\boldsymbol{\theta}) ) } \\
        & = l_{N-1}(\boldsymbol{\theta}) + a + b,
    \end{split}
\end{equation}
where
\begin{equation}
    \begin{split} 
        & a = e^{\sum_{k=1}^{N-1} \boldsymbol{1}(\boldsymbol{\eta}_N^{\theta} -  W(\mathrm{\bold{x}}_k,t_{k};\boldsymbol{\theta}) )} e^{ \boldsymbol{1}(\boldsymbol{\eta}_N^{\theta} - W(\mathrm{\bold{x}}_N,t_{N};\boldsymbol{\theta}) ) } = e \cdot e^{I^{N-1}(\boldsymbol{\eta}_N^{\theta};\boldsymbol{\theta})}  \\
        & b = - e^{\sum_{k=1}^{N-1} \boldsymbol{1}(\boldsymbol{\eta}_N^{\theta} -  W(\mathrm{\bold{x}}_k,t_{k};\boldsymbol{\theta}) )} = -e^{I^{N-1}(\boldsymbol{\eta}_N^{\theta_0};\boldsymbol{\theta})}\\
        & \boldsymbol{\eta}_N^{\theta} = \operatorname{round}(W(\mathrm{\bold{x}}_{N},t_{N};\boldsymbol{\theta})).
    \end{split}
\end{equation}
Thus, the bounds are
\begin{equation}
    \begin{split} 
        & \overline{L_{N}} = \overline{L_{N-1}} + (e-1) e^{Q}  \\
        & \underline{L_{N}} = \underline{L_{N-1}} + (e-1) e^{I^{N-1}(\boldsymbol{\eta}_N^{\theta_0};\boldsymbol{\theta}_0)}.
    \end{split}
\end{equation}

\subsection{Sum of Suppressed Accumulations (SoSA)}

\begin{equation}
	\begin{split}
        l_{N}(\boldsymbol{\theta}) 
        & = \sum_{\mathbf{p}_{ij}\in\mathcal{P}}e^{-I(\mathbf{p}_{ij};\boldsymbol{\theta})\cdot\delta} \\
        & = \sum_{\mathbf{p}_{ij}\in\mathcal{P}}e^{ -\delta \cdot \sum_{k=1}^{N} \boldsymbol{1}(\textbf{p}_{ij} -  W(\mathrm{\bold{x}}_k,t_{k};\boldsymbol{\theta}) ) } \\
        & = \sum_{\mathbf{p}_{ij}\in\mathcal{P}}e^{-\delta \cdot \sum_{k=1}^{N-1} \boldsymbol{1}(\textbf{p}_{ij} -  W(\mathrm{\bold{x}}_k,t_{k};\boldsymbol{\theta}) ) -\delta \cdot \boldsymbol{1}(\textbf{p}_{ij} - W(\mathrm{\bold{x}}_N,t_{N};\boldsymbol{\theta}) ) } \\
        & = l_{N-1}(\boldsymbol{\theta}) + a + b,
    \end{split}
\end{equation}
where
\begin{equation}
    \begin{split} 
        & a = e^{ -\delta \cdot \sum_{k=1}^{N-1} \boldsymbol{1}(\boldsymbol{\eta}_N^{\theta} -  W(\mathrm{\bold{x}}_k,t_{k};\boldsymbol{\theta}) )} e^{  -\delta \cdot \boldsymbol{1}(\boldsymbol{\eta}_N^{\theta} - W(\mathrm{\bold{x}}_N,t_{N};\boldsymbol{\theta}) ) } = e^{ -\delta} \cdot e^{ -\delta \cdot I^{N-1}(\boldsymbol{\eta}_N^{\theta};\boldsymbol{\theta})}  \\
        & b = - e^{ -\delta \cdot \sum_{k=1}^{N-1} \boldsymbol{1}(\boldsymbol{\eta}_N^{\theta} -  W(\mathrm{\bold{x}}_k,t_{k};\boldsymbol{\theta}) )} = -e^{ -\delta \cdot I^{N-1}(\boldsymbol{\eta}_N^{\theta};\boldsymbol{\theta})}.
    \end{split}
\end{equation}
Thus, the bounds are
\begin{equation}
    \begin{split} 
        & \overline{L_{N}} = \overline{L_{N-1}} + (e^{-\delta}-1) e^{-\delta \cdot Q} \\
        & \underline{L_{N}} = \underline{L_{N-1}} + (e^{-\delta}-1) e^{-\delta \cdot I^{N-1}(\boldsymbol{\eta}_N^{\theta_0};\boldsymbol{\theta}_0)}.
    \end{split}
\end{equation}

SoEaS and SoSAaS are combination loss functions, the bounds are also a combination, so we omit the derivation here. 


\section{Application to Visual Odometry with a downward-facing Event Camera}

\subsection{Bounding Box Definition}

We have
\begin{eqnarray}
  \mathbf{x}'_k 
  & = & \left[ \begin{matrix}
      - [y_k - v_0 + s \frac{f}{d}] \sin(\omega t)
      + [x_k - u_0 - \frac{f}{d} (\frac{v}{\omega})] \cos(\omega t)
      + \frac{f}{d} (\frac{v}{\omega}) + u_0 \\
      [x_k - u_0 - \frac{f}{d} (\frac{v}{\omega})] \sin(\omega t)
      + [y_k - v_0 + s \frac{f}{d}] \cos(\omega t)
      - s \frac{f}{d} + v_0
  \end{matrix} \right] \nonumber \\
  & = & \left[ \begin{matrix}
      a_x + b_x + c_x + u_0 \\
      a_y + b_y + c_y - s \frac{f}{d} + v_0
  \end{matrix} \right],  
\end{eqnarray}

where
\begin{eqnarray}
    a_x &=& - [y_k - v_0 + s \frac{f}{d}] \sin(\omega t), \nonumber \\
    b_x  &=& [x_k - u_0 ] \cos(\omega t) , \nonumber \\
    c_x  &=& \frac{f}{d} (\frac{v}{\omega})[1- \cos(\omega t)] , \nonumber \\
    a_y &=& [x_k - u_0 ] \sin(\omega t), \nonumber \\
    b_y &=&  - \frac{f}{d} (\frac{v}{\omega}) \sin(\omega t), \nonumber \\
    c_y &=&  [y_k - v_0 + s \frac{f}{d}] \cos(\omega t).
\end{eqnarray}
The bounding box $\mathcal{P}_k^{\boldsymbol{\Theta}}$ is found by bounding the values of $x^{\prime}_{k}$ and $y^{\prime}_{k}$ over the intervals $\omega\in\left[\omega_{\mathrm{min}};\omega_{\mathrm{max}}\right]$ and $v\in\left[v_{\mathrm{min}};v_{\mathrm{max}}\right]$. Here we only consider the case $\mathrm{abs}(\omega t) < \pi/2$. The bounding box is easily achieved if simply considering the monotonicity and different cases. There are 17 cases in total. One case is when $\omega = 0$. Given the Ackermann motion model, we then obtain 
\small
\begin{eqnarray}
    \underline{x^{\prime}_{k}} &=& x_k ,\  \overline{x^{\prime}_{k}}  = -x_k , \nonumber \\
    \underline{y^{\prime}_{k}} &=& y_k + \frac{f}{d} v_{\mathrm{min}} t, \  \overline{y^{\prime}_{k}} =  y_k + \frac{f}{d} v_{\mathrm{max}} t.
\end{eqnarray}
\normalsize
The other 16 cases are based on the monotonicity of functions. For example, if $\omega_{\mathrm{min}} \geq 0$, $v_{\mathrm{min}} \geq 0$ and $x_k \geq u_0,\ y_k \geq v_0 - s\frac{f}{d}$, the lower bound of $x^{\prime}_{k}$ is
\small
\begin{eqnarray}
    \underline{x^{\prime}_{k}} &=& \min_{\omega} a_x + \min_{\omega} b_x + \min_{\omega,v} c_x + u_0 \ \text{, with} \\
    \min_{\omega} a_x & \geq & - [y_k - v_0 + s \frac{f}{d}] \sin(\omega_{\mathrm{max}} t), \nonumber \\
    \min_{\omega} b_x & \geq & [x_k - u_0 ] \cos(\omega_{\mathrm{max}} t) , \nonumber \\
    \min_{\omega,v} c_x & \geq & \frac{f}{d} (\frac{v_{\mathrm{min}}}{\omega_{\mathrm{max}}})[1- \cos(\omega_{\mathrm{max}} t)].
\end{eqnarray}
\normalsize
Table~\ref{tab:bounding box cases} lists $\underline{x^{\prime}_{k}}$ and $\underline{y^{\prime}_{k}}$ with $\omega$ and $v$ arguments when the search space is $\omega_{\mathrm{min}} > 0$. Meanwhile $\overline{x^{\prime}_{k}}$ and $\overline{y^{\prime}_{k}}$ are obtained by $\omega_{\mathrm{min}}$ against $\omega_{\mathrm{max}}$, and $v_{\mathrm{min}}$ against $v_{\mathrm{max}}$. The other 8 cases with $\omega_{\mathrm{max}} < 0$ are derived by a similar strategy.
\begin{table}[]
\caption{Bounding Box Cases}
\centering
\begin{tabular}{|c|c|c|c|c|c|c|c|c|}
\hline
\multicolumn{3}{|c|}{\multirow{2}{*}{Conditions}}  & \multicolumn{3}{c|}{ $\underline{x^{\prime}_{k}}$}                         & \multicolumn{3}{c|}{ $\underline{y^{\prime}_{k}}$} \\ \cline{4-9} 
\multicolumn{3}{|c|}{}    & $a_x$ & $b_x$   & $c_x$ & $a_y$ & $b_y$ & $c_y$ 
\\ \hline
\multirow{16}*{\ $\omega_{\mathrm{min}} > 0$\ } & \multicolumn{1}{c|}{\multirow{8}{*}{$\ \ v_{\mathrm{min}} \geq 0$\ \ }} & \begin{tabular}[c]{@{}l@{}}$\ x_k \geq u_0$,\\ $\ y_k \geq v_0 - s\frac{f}{d}$\ \ \end{tabular}   &\ $\omega_{\mathrm{max}}\ $     &\ $\omega_{\mathrm{max}}\ $    &\ \begin{tabular}[c]{@{}l@{}}$\omega_{\mathrm{max}}$,\ \ \\ $v_{\mathrm{min}}$\end{tabular}                     &\ $\omega_{\mathrm{min}}$\ \ &\ \begin{tabular}[c]{@{}l@{}}$\omega_{\mathrm{min}}$,\ \ \\ $v_{\mathrm{max}}$\end{tabular}        & \ $\omega_{\mathrm{max}}\ $       \\ \cline{3-9} 
                  & \multicolumn{1}{c|}{}                  & \begin{tabular}[c]{@{}l@{}}$x_k < u_0$,\\ $y_k \geq v_0 - s\frac{f}{d}$\end{tabular}  & $\omega_{\mathrm{max}}$     &                      $\omega_{\mathrm{min}}$ & \begin{tabular}[c]{@{}l@{}} $\omega_{\mathrm{max}}$,\\ $v_{\mathrm{min}}$\end{tabular}                      & $\omega_{\mathrm{max}}$         & \begin{tabular}[c]{@{}l@{}}$\omega_{\mathrm{min}}$,\\ $v_{\mathrm{max}}$\end{tabular}        &  $\omega_{\mathrm{max}}$       \\ \cline{3-9} 
                  & \multicolumn{1}{c|}{}                  & \begin{tabular}[c]{@{}l@{}}$x_k < u_0$,\\ $y_k < v_0 - s\frac{f}{d}$\end{tabular}  &  $\omega_{\mathrm{min}}$    &                      $\omega_{\mathrm{min}}$ &   \begin{tabular}[c]{@{}l@{}}$\omega_{\mathrm{max}}$,\\ $v_{\mathrm{min}}$\end{tabular}                     &  $\omega_{\mathrm{max}}$        &  \begin{tabular}[c]{@{}l@{}}$\omega_{\mathrm{min}}$,\\ $v_{\mathrm{max}}$\end{tabular}                           &  $\omega_{\mathrm{min}}$         \\ \cline{3-9} 
                  & \multicolumn{1}{c|}{}                  & \begin{tabular}[c]{@{}l@{}}$x_k \geq u_0$,\\ $y_k < v_0 - s\frac{f}{d}$\end{tabular}  & $\omega_{\mathrm{min}}$     &  $\omega_{\mathrm{max}}$           &   \begin{tabular}[c]{@{}l@{}}$\omega_{\mathrm{max}}$,\\ $v_{\mathrm{min}}$\end{tabular}                    &$\omega_{\mathrm{min}}$          &  \begin{tabular}[c]{@{}l@{}}$\omega_{\mathrm{min}}$,\\ $v_{\mathrm{max}}$\end{tabular}        & $\omega_{\mathrm{min}}$        \\ \cline{2-9} 
 & { \multirow{8}{*}{$v_{\mathrm{min}} < 0$}} & \begin{tabular}[c]{@{}l@{}}$x_k \geq u_0$,\\ $y_k \geq v_0 - s\frac{f}{d}$\end{tabular}   & $\omega_{\mathrm{max}}$     & $\omega_{\mathrm{max}}$    &  \begin{tabular}[c]{@{}l@{}}$\omega_{\mathrm{min}}$,\\ $v_{\mathrm{min}}$\end{tabular}                     & $\omega_{\mathrm{min}}$      & \begin{tabular}[c]{@{}l@{}}$\omega_{\mathrm{min}}$,\\ $v_{\mathrm{max}}$\end{tabular}        &  $\omega_{\mathrm{max}}$       \\ \cline{3-9} 
                  & \multicolumn{1}{c|}{}                  & \begin{tabular}[c]{@{}l@{}}$x_k < u_0$,\\ $y_k \geq v_0 - s\frac{f}{d}$\end{tabular}  & $\omega_{\mathrm{max}}$     &                      $\omega_{\mathrm{min}}$ & \begin{tabular}[c]{@{}l@{}} $\omega_{\mathrm{min}}$,\\ $v_{\mathrm{min}}$\end{tabular}                      & $\omega_{\mathrm{max}}$         & \begin{tabular}[c]{@{}l@{}}$\omega_{\mathrm{max}}$,\\ $v_{\mathrm{max}}$\end{tabular}        &  $\omega_{\mathrm{max}}$       \\ \cline{3-9} 
                  & \multicolumn{1}{c|}{}                  & \begin{tabular}[c]{@{}l@{}}$x_k < u_0$,\\ $y_k < v_0 - s\frac{f}{d}$\end{tabular}  &  $\omega_{\mathrm{min}}$    &                      $\omega_{\mathrm{min}}$ &   \begin{tabular}[c]{@{}l@{}}$\omega_{\mathrm{min}}$,\\ $v_{\mathrm{min}}$\end{tabular}                     &  $\omega_{\mathrm{max}}$        &  \begin{tabular}[c]{@{}l@{}}$\omega_{\mathrm{max}}$,\\ $v_{\mathrm{max}}$\end{tabular}                           &  $\omega_{\mathrm{min}}$         \\ \cline{3-9} 
                  & \multicolumn{1}{c|}{}                  & \begin{tabular}[c]{@{}l@{}}$x_k \geq u_0$,\\ $y_k < v_0 - s\frac{f}{d}$\end{tabular}  & $\omega_{\mathrm{min}}$     &  $\omega_{\mathrm{max}}$           &   \begin{tabular}[c]{@{}l@{}}$\omega_{\mathrm{min}}$,\\ $v_{\mathrm{min}}$\end{tabular}                    &$\omega_{\mathrm{min}}$          &  \begin{tabular}[c]{@{}l@{}}$\omega_{\mathrm{max}}$,\\ $v_{\mathrm{max}}$\end{tabular}        & $\omega_{\mathrm{min}}$        \\ \hline
\end{tabular}
\label{tab:bounding box cases}
\end{table}

\subsection{Application to Real Data and Comparison against Alternatives}

\textbf{Various Textures}: To further analyse the robustness, we test our algorithm on datasets collected with various textures. Figure~\ref{fig:carpet_and_poster} presents frames from two further datasets named \textit{Carpet} and \textit{Poster}. The \textit{Carpet} sequences are collected on a carpet with non-repetitive almost featureless texture, while the \textit{Poster} sequences are collected on a poster with characters and figures for which it is easy to extract features. The estimated errors are summarised on the left of Figure ~\ref{fig:carpet_and_poster}. As can be observed, our algorithm continues have similar accuracy for the various textures in the datasets.
\begin{figure}[t]
\subfigure
{
\begin{minipage}{0.6\textwidth}
\includegraphics[width=0.188\textwidth]{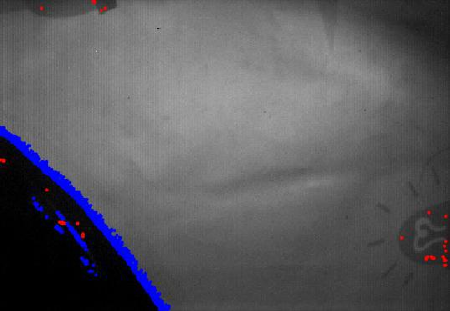}
\includegraphics[width=0.188\textwidth]{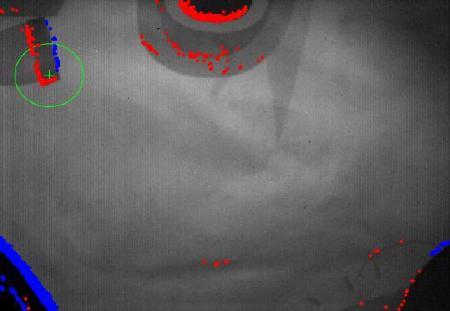}
\includegraphics[width=0.188\textwidth]{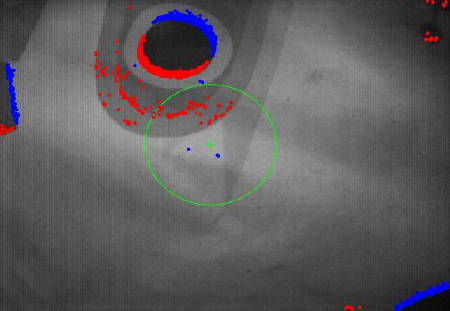}
\includegraphics[width=0.188\textwidth]{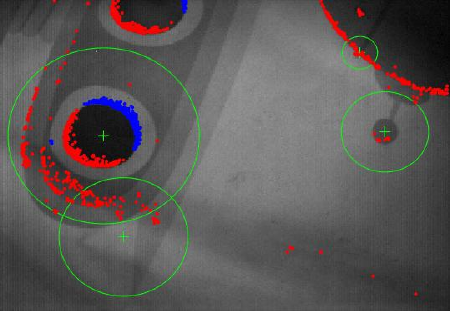}
\includegraphics[width=0.188\textwidth]{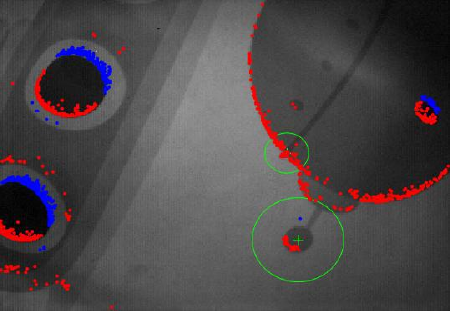} \\
\includegraphics[width=0.188\textwidth]{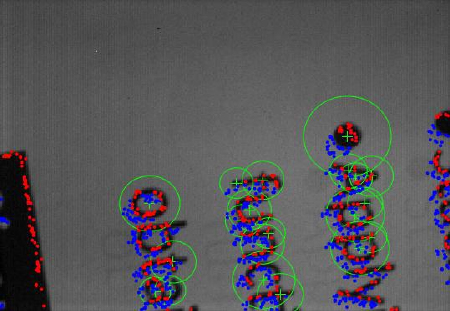}
\includegraphics[width=0.188\textwidth]{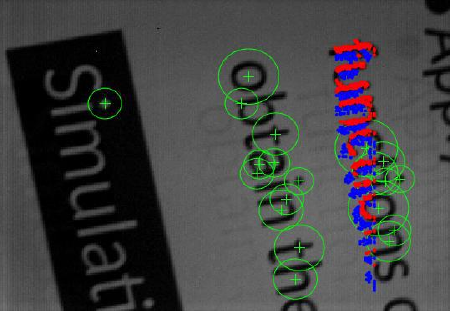}
\includegraphics[width=0.188\textwidth]{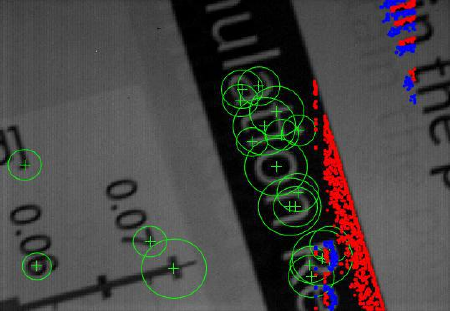}
\includegraphics[width=0.188\textwidth]{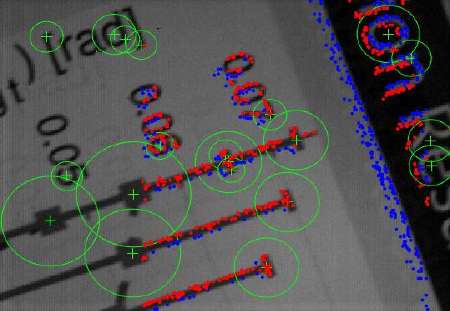}
\includegraphics[width=0.188\textwidth]{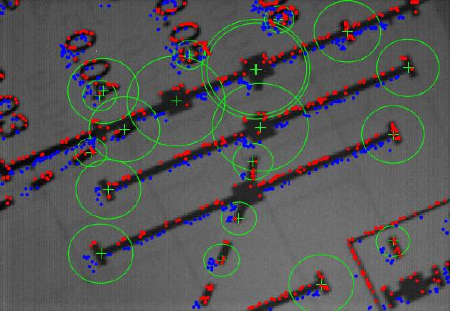}
\end{minipage}
}
\hspace{0\textwidth}
\begin{minipage}{0.20\textwidth}
{
\renewcommand\arraystretch{1.2}
\begin{tabular}{|c|c|c|}
\hline
\bf Scene & \bf\ \begin{tabular}[c]{@{}c@{}} w{[}$^{\circ}$/s{]}\end{tabular}\ \ & \bf\ \begin{tabular}[c]{@{}c@{}} v{[}m/s{]}\end{tabular}\ \ \\ \hline
\bf\ Carpet\ \ & 4.730 & 0.034 \\ \hline
\bf\ Poster\ \ & 3.122 & 0.030 \\
\hline
\end{tabular}
}
\end{minipage}
\caption{Frames from dataset \textit{Carpet} (first row) and \textit{Poster} (second row) and RMS errors for the different textures. }
\label{fig:carpet_and_poster}
\end{figure}
